\documentclass{article}

\PassOptionsToPackage{numbers}{natbib}
\usepackage[final]{neurips_2022}

\usepackage[utf8]{inputenc} % allow utf-8 input
\usepackage[T1]{fontenc}    % use 8-bit T1 fonts
\usepackage{hyperref}       % hyperlinks
\usepackage{url}            % simple URL typesetting
\usepackage{booktabs}       % professional-quality tables
\usepackage{amsfonts}       % blackboard math symbols
\usepackage{nicefrac}       % compact symbols for 1/2, etc.
\usepackage{microtype}      % microtypography
\usepackage{xcolor,colortbl}         % colors
\usepackage{subcaption}
\usepackage{graphicx}
\usepackage{amssymb,amsmath,amsthm}
\usepackage{wrapfig}
\usepackage{caption}
\usepackage{multirow}
\usepackage{enumitem}
\usepackage{caption}
\newtheorem{theorem}{Theorem}
\newtheorem{definition}{Definition}
\newtheorem{lemma}{Lemma}
\newtheorem{proposition}{Proposition}
\newtheorem{corollary}{Corollary}
\newtheorem{property}{Property}

\newcommand{\first}[1]{\textcolor{red}{\textbf{#1}}}
\newcommand{\second}[1]{\textbf{#1}}

\title{How Powerful are $K$-hop Message Passing Graph Neural Networks}

\author{
    Jiarui Feng$^{1,2}$~~~Yixin Chen$^{1}$~~~Fuhai Li$^{2}$~~~Anindya Sarkar$^{1}$~~~Muhan Zhang$^{3,4}$\\
    \texttt{\{feng.jiarui, fuhai.li, anindya\}@wustl.edu,}\\
    \texttt{chen@cse.wustl.edu,~muhan@pku.edu.cn}\\${}^1$Department of CSE, Washington University in St. Louis\\
${}^2$Institute for Informatics, Washington University School of Medicine\\
${}^3$Institute for Artificial Intelligence, Peking University\\
${}^4$Beijing Institute for General Artificial Intelligence\\}

\begin{document}
\maketitle

\begin{abstract}
The most popular design paradigm for Graph Neural Networks (GNNs) is 1-hop message passing---aggregating information from 1-hop neighbors repeatedly. However, the expressive power of 1-hop message passing is bounded by the Weisfeiler-Lehman (1-WL) test. Recently, researchers extended 1-hop message passing to $K$-hop message passing by aggregating information from $K$-hop neighbors of nodes simultaneously. However, there is no work on analyzing the expressive power of $K$-hop message passing. In this work, we theoretically characterize the expressive power of $K$-hop message passing. Specifically, we first formally differentiate two different kernels of $K$-hop message passing which are often misused in previous works. We then characterize the expressive power of $K$-hop message passing by showing that it is more powerful than 1-WL and can distinguish almost all regular graphs. Despite the higher expressive power, we show that $K$-hop message passing still cannot distinguish some simple regular graphs and its expressive power is bounded by 3-WL. To further enhance its expressive power, we introduce a KP-GNN framework, which improves $K$-hop message passing by leveraging the peripheral subgraph information in each hop. We show that KP-GNN can distinguish many distance regular graphs which could not be distinguished by previous distance encoding or 3-WL methods. Experimental results verify the expressive power and effectiveness of KP-GNN. KP-GNN achieves competitive results across all benchmark datasets.
\end{abstract}

\section{Introduction}
Currently, most existing graph neural networks (GNNs) follow the \textit{message passing} framework, which iteratively aggregates information from the neighbors and updates the representations of nodes. It has shown superior performance on graph-related tasks~\citep{kipf2017semisupervised,duvenaud2015convolutional,hamilton2017inductive,veličković2018graph,li2015gated,zhang2018end,xu2018powerful} comparing to traditional graph embedding techniques~\citep{grover2016node2vec,perozzi2014deepwalk}. However, as the procedure of message passing is similar to the 1-dimensional Weisfeiler-Lehman (1-WL) test~\citep{weisfeiler1968reduction}, the expressive power of message passing GNNs is also bounded by the 1-WL test~\citep{xu2018powerful, morris2019weisfeiler}. Namely, GNNs cannot distinguish two non-isomorphic graph structures if the 1-WL test fails.  

In normal message passing GNNs, the node representation is updated by the direct neighbors of the node, which are called 1-hop neighbors. Recently, some works extend the notion of message passing into $K$-hop message passing~\citep{abu2019mixhop,nikolentzos2020k,wang2021multihop,chien2021adaptive,brossard2020graph}. \textbf{$K$-hop message passing} is a type of message passing where the node representation is updated by aggregating information from not only 1st hop but all the neighbors within $K$ hops of the node. However, there is no work on theoretically characterizing the expressive power of GNNs with $K$-hop message passing, e.g., whether it can improve the 1-hop message passing or not and to what extent it can.

In this work, we theoretically characterize the expressive power of $K$-hop message passing GNNs. Specifically, 1) we formally distinguish two different kernels of the $K$-hop neighbors, which are often misused in previous works. The first kernel is based on whether the node can be reached within $k$ steps of the graph diffusion process, which is used in GPR-GNN~\citep{chien2021adaptive} and MixHop~\citep{abu2019mixhop}. The second one is based on the shortest path distance of $k$, which is used in GINE+~\citep{brossard2020graph} and Graphormer~\cite{ying2021do}. Further, we show that different kernels of $K$-hop neighbors will result in \textbf{different expressive power} of $K$-hop message passing. 2) We show that \textbf{$K$-hop message passing is strictly more powerful than 1-hop message passing and can distinguish almost all regular graphs. 3) However, it still failed in distinguishing some simple regular graphs, no matter which kernel is used, and its expressive power is bounded by 3-WL}. This motivates us to improve $K$-hop message passing further.

Here, we introduce \textbf{KP-GNN}, a new GNN framework with $K$-hop message passing, which significantly improves the expressive power of standard $K$-hop message passing GNNs. In particular, during the aggregation of neighbors in each hop, KP-GNN not only aggregates neighboring nodes in that hop but also aggregates the \textbf{peripheral subgraph} (subgraph induced by the neighbors in that hop). This additional information helps the KP-GNN to learn more expressive local structural features around the node. We further show that KP-GNN can distinguish many distance regular graphs with a proper encoder for the peripheral subgraph. The proposed KP-GNN has several additional advantages. First, it can be applied to most existing $K$-hop message-passing GNNs with only slight modification. Second, it only adds little computational complexity to standard $K$-hop message passing. We demonstrate the effectiveness of the KP-GNN framework through extensive experiments on both simulation and real-world datasets.

\section{$K$-hop message passing and its expressive power}
\subsection{Notations}
Denote a graph as $G=(V,E)$, where $V=\{1,2,...,n\}$ is the node set and $E\subseteq V\times V$ is the edge set. Meanwhile, denote $A\in \{0,1\}^{n\times n}$ as the adjacency matrix of graph $G$. Denote $x_v$ as the feature vector of node $v$ and denote $e_{uv}$ as the feature vector of the edge from $u$ to $v$. Finally, we denote $Q^{1}_{v,G}$ as the set of 1-hop neighbors of node $v$ in graph $G$ and $\mathcal{N}^{1}_{v,G}$=$Q^{1}_{v,G} \cup \{v\}$. Note that when we say $K$-hop neighbors of node $v$, we mean \textbf{all} the neighbors that have distance from node $v$ less than or equal to $K$. In contrast, $k$-th hop neighbors mean the neighbors with \textbf{exactly} distance $k$ from node $v$. The definition of distance will be discussed in section \ref{sec:khop}.

\subsection{1-hop message passing framework}
Currently, most existing GNNs are designed based on 1-hop message passing framework \citep{gilmer2017neural}. Denote $h^{l}_{v}$ as the output representation of node $v$ at layer $l$ and $h^{0}_{v}=x_v$. Briefly, given a graph $G$ and a 1-hop message passing GNN, at layer $l$ of the GNN, $h^{l}_{v}$ is computed by $h^{l-1}_{v}$ and $\{\!\!\{h^{l-1}_{u} ~|~ u \in  Q^{1}_{v,G}\}\!\!\}$:
\begin{align}
\label{eq:1hop_mp}
\begin{aligned}
m^{l}_{v}=\text{MES}^{l}(\{\!\!\{(h^{l-1}_{u},e_{uv})|u\in Q^{1}_{v,G}\}\!\!\}),~~
h^{l}_{v}=\text{UPD}^{l}(m^{l}_{v},h^{l-1}_{v}),
\end{aligned}
\end{align}
where $m^{l}_{v}$ is the message to node $v$ at layer $l$, $\text{MES}^l$ and $\text{UPD}^l$ are message and update functions at layer $l$ respectively. After $L$ layers of message passing, $h^{L}_{v}$ is used as the final representation of node $v$. Such a representation can be used to conduct node-level tasks like node classification and node regression. To get the graph representation, a readout function is used:
\begin{align}
    h_{G}=\text{READOUT}(\{\!\!\{h^{L}_{v}|v\in V\}\!\!\}),
\end{align}
where $\text{READOUT}$ is the readout function for computing the final graph representation. Then $h_{G}$ can be used to conduct graph-level tasks like graph classification and graph regression.

\subsection{$K$-hop message passing framework}\label{sec:khop}
The 1-hop message passing framework can be directly generalized to $K$-hop message passing, as it shares the same message and update mechanism. The difference is that independent message and update functions can be employed for each hop. Meanwhile, a combination function is needed to combine the results from different hops into the final node representation at this layer. First, we differentiate two different kernels of $K$-hop neighbors, which are interchanged and misused in previous research. 

The first kernel of $K$-hop neighbors is \textit{shortest path distance (spd) kernel}. Namely, the $k$-th hop neighbors of node $v$ in graph $G$ is the set of nodes with the shortest path distance of $k$ from $v$. 
\begin{definition}
\label{def:spd}
 \textit{For a node $v$ in graph $G$, the $K$-hop neighbors $\mathcal{N}^{K,spd}_{v,G}$ of $v$ based on shortest path distance kernel is the set of nodes that have the shortest path distance from node $v$ \textbf{less than or equal to $K$}. We further denote $Q^{k,spd}_{v,G}$ as the set of nodes in $G$ that are \textbf{exactly} the $k$-th hop neighbors (with shortest path distance of exactly $k$) and $\mathcal{N}^{0,spd}_{v,G}=Q^{0,spd}_{v,G}=\{v\}$ is the node itself. }
\end{definition}
The second kernel of the $K$-hop neighbors is based on \textit{graph diffusion (gd)}. 
%Given the adjacency matrix $A$ of graph $G$, the $K$-hop neighbors of node $v$ in graph $G$ is the set of nodes that can be reached within $K$ steps of random walk diffusion.
\begin{definition}
\label{def:gd}
\textit{For a node $v$ in graph $G$, the $K$-hop neighbors $\mathcal{N}^{K,gd}_{v,G}$ of $v$ based on graph diffusion kernel is the set of nodes that can diffuse information to node $v$ \textbf{within the number of random walk diffusion steps} $K$ and the diffusion kernel $A$ (adjacency matrix). We further denote $Q^{k,gd}_{v,G}$ as the set of nodes in $G$ that are \textbf{exactly} the $k$-th hop neighbors (nodes that can diffuse information to node $v$ with $k$ diffusion steps) and $\mathcal{N}^{0,gd}_{v,G}=Q^{0,gd}_{v,G}=\{v\}$ is the node itself.}
\end{definition}
Note that a node can be a $k$-th hop neighbor of $v$ for multiple $k$ based on the graph diffusion kernel, but it can only appear in one hop for the shortest path distance kernel. We include more discussions of $K$-hop kernels in Appendix~\ref{app:khop_more}. Next, we define the $K$-hop message passing framework as follows:
\begin{align}
\label{eq:khop_mp}
\begin{aligned}
m^{l,k}_{v}=\ \text{MES}^{l}_{k}(\{\!\!\{(h^{l-1}_{u},e_{uv})|u\in Q^{k,t}_{v,G}&)\}\!\!\}),~~
h^{l,k}_{v}=\ \text{UPD}^{l}_{k}(m^{l,k}_{v},h^{l-1}_{v}),\\
h^{l}_{v}=\ \text{COMBINE}^l(\{\!\!\{h^{l,k}_{v}&|k=1,2,...,K\}\!\!\}),
\end{aligned}
\end{align}
where $t \in\{spd,gd\}$, $spd$ is the shortest path distance kernel and $gd$ is the graph diffusion kernel. Here, for each hop, we can apply unique $\text{MES}$ and $\text{UPD}$ functions. Note that for $k>1$, there may not exist the edge feature $e_{uv}$ as nodes are not directly connected. But we leave it here since we can use other types of features to replace it like path encoding. We further discuss it in Appendix~\ref{app:implementation}. Compared to the 1-hop message passing framework described in Equation~(\ref{eq:1hop_mp}), the $\text{COMBINE}$ function is introduced to combine the representations of node $v$ at different hops. It is easy to see that a $L$ layer 1-hop message passing GNNs is actually a $L$ layer $K$-hop message passing GNNs with $K=1$. We include more discussions of $K$-hop message passing GNNs in Appendix~\ref{app:khop_more}. To aid further analysis, we also prove that $K$-hop message passing can injectively encode the neighbor representations at different hops into $h_v^l$ in Appendix~\ref{app:injective}.

%Given the formal definition of the $K$-hop message passing framework, now we are ready to characterize its representation power. 

\subsection{Expressive power of $K$-hop message passing framework}
% \begin{wrapfigure}[]{L}{0.45\textwidth}
%\vspace{-10pt}
\begin{figure*}[t]
\centering
% \vspace{5pt}
\includegraphics[width=0.95\textwidth]{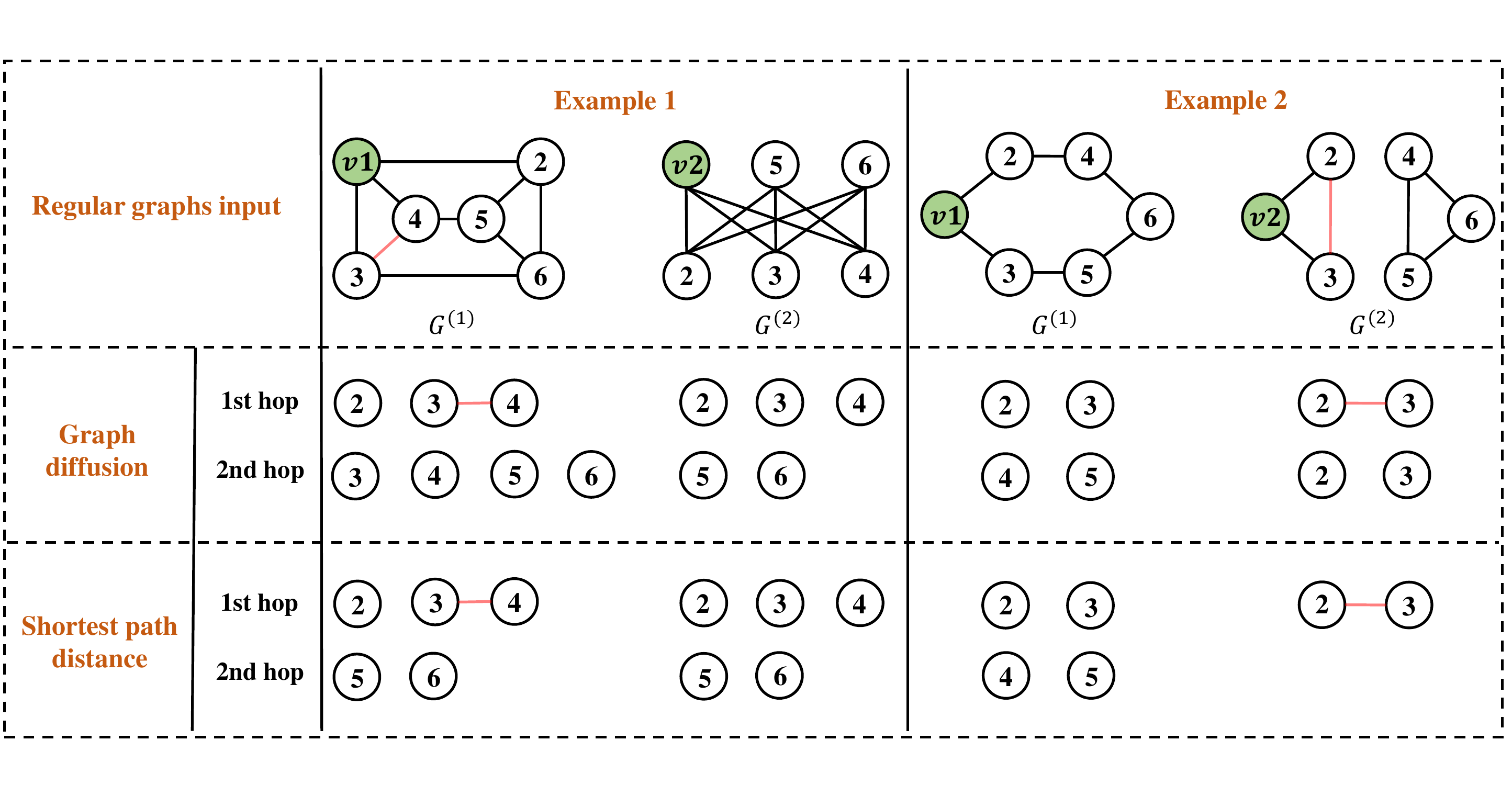}
\vspace{-5pt}
\caption{Here are two pairs of non-isomorphic regular graphs. With 2-hop message passing, example 1 can be distinguished by the graph diffusion kernel, and example 2 can be distinguished by the shortest path distance kernel. However, both two examples become indistinguishable if we switch the kernel. Finally, both two examples can be distinguished by adding peripheral subgraph information.}
 \vspace{-15pt}
\label{fig:regular1}
\end{figure*}
% \end{wrapfigure}

 In this section, we theoretically analyze the expressive power of $K$-hop message passing. We assume there is no edge feature and all nodes in the graph have the same feature, which means that GNNs can only distinguish two nodes with the local structure of nodes. Note that including node features only increases the expressive power of GNNs as nodes/graphs are more easily to be discriminated. It has been proved that the expressive power of 1-hop message passing is bounded by the 1-WL test on discriminating non-isomorphic graphs~\citep{xu2018powerful,morris2019weisfeiler}. In this section, We show that the $K$-hop message passing is strictly more powerful than the 1-WL test when $K>1$. Across the analysis, we utilize regular graphs as examples to illustrate our theorems since they cannot be distinguished using either 1-hop message passing or the 1-WL test. To begin the analysis, we first define \textit{proper $K$-hop message passing GNNs}.

\begin{definition}
\textit{A proper $K$-hop message passing GNN is a GNN model where the message, update, and combine functions are all injective given the input from a countable space.}
\end{definition}
A proper $K$-hop message passing GNN is easy to find due to the universal approximation theorem~\citep{Cybenko1989ApproximationBS} of neural network and the Deep Set for set operation~\citep{zaheer2017deep}. In the latter sections, by default, all mentioned $K$-hop message passing GNNs are proper. Next, we introduce \textit{node configuration}: 

\begin{definition}
\textit{The node configuration of node $v$ in graph $G$ within $K$ hops under $t$ kernel is a list $A^{K,t}_{v,G}=(a^{1,t}_{v,G},a^{2,t}_{v,G},...,a^{K,t}_{v,G})$, where $a^{i,t}_{v,G}=|Q^{i,t}_{v,G}|$ is the number of $i$-th hop neighbors of node $v$.}
\end{definition}
When we say two node configurations $A^{K,t}_{v_1,G^{(1)}}$ and $A^{K,t}_{v_2,G^{(2)}}$ are equal, we mean that these two lists are component-wise equal to each other. Now, we state the first proposition:
\begin{proposition}
\label{pro:khop_power}
\textit{A proper $K$-hop message passing GNN is strictly more powerful than 1-hop message passing GNNs when $K>1$.}
\end{proposition}
To see why this is true, we first discuss how node configuration relates to the first layer of message passing. In the first layer of $K$-hop message passing, each node aggregates neighbors from up to $K$ hops. As each node has the same node label, an injective message function can only know how many neighbors at each hop, which is exactly the node configuration. In other words, \textbf{The first layer of $K$-hop message passing is equivalent to inject node configuration to each node label}. When $K=1$, the node configuration of $v_1$ and $v_2$ are $d_{v_1,G^{(1)}}$ and $d_{v_2,G^{(2)}}$, where $d_{v,G}$ is the node degree of $v$. After $L$ layers, GNNs can only get the node degree information of each node within $L$ hops of node $v$. Then, it is straightforward to see why these GNNs cannot distinguish any $n$-sized $r$-regular graph, as each node in the regular graph has the same degree.

Next, when $K>1$, the $K$-hop message passing is at least equally powerful as 1-hop message passing since node configuration up to $K$ hop includes all the information the 1 hop has. To see why it is more powerful, we use two examples to illustrate it. The first example is shown in the left part of Figure~\ref{fig:regular1}. Suppose here we use graph diffusion kernel and we want to learn the representation of node $v_1$ and node $v_2$ in the two graphs. We know that the 1-hop GNNs produce the same representation for two nodes as they are both nodes in 6-sized 3-regular graphs. However, it is easy to see that $v_1$ and $v_2$ have different local structures and should have different representations. Instead, if we use the 2-hop message passing with the graph diffusion kernel, we can easily distinguish two nodes by checking the 2nd hop neighbors of the node, as node $v_1$ has four 2nd hop neighbors but node $v_2$ only has two 2nd hop neighbors. The second example is shown in the right part of Figure~\ref{fig:regular1}. Two graphs in the example are still regular graphs. Suppose here we use shortest path distance kernel, node $v_1$ and $v_2$ have different numbers of 2nd hop neighbors and thus will have different representations by performing 2-hop message passing. These two examples convincingly demonstrate that the $K$-hop message passing with $K>1$ can have better expressive power than $K=1$. To further study the expressive power of $K$-hop message passing on regular graphs, we show the following result:
\begin{theorem}
\label{thm:khop_regular}
Consider all pairs of $n$-sized $r$-regular graphs, let $ 3 \leq r < (2log2n)^{1/2}$ and $\epsilon$ be a fixed constant. With at most $K=\lfloor(\frac{1}{2}+\epsilon)\frac{\log{2n}}{\log{(r-1)}} \rfloor$, there exists a 1 layer $K$-hop message passing GNN using the shortest path distance kernel that distinguishes almost all $1-o(n^{-1/2})$ such pairs of graphs.
\end{theorem}
We include the proof and simulation results in Appendix~\ref{app:khop_regular}. Theorem~\ref{thm:khop_regular} shows that even with 1 layer and a modest $K$, $K$-hop GNNs are powerful enough to distinguish almost all regular graphs.

Finally, we characterize the existing $K$-hop methods with the proposed $K$-hop message passing framework. Specifically, we show that 1) the expressive power of $K$ layer GINE~\citep{brossard2020graph} is bounded by $K$ layer $K$-hop message passing with the shortest path distance kernel. 2) The expressive power of the Graphormer~\citep{ying2021do} is equal to $K$-hop GNNs with the shortest path distance kernel and infinity $K$. 3) For spectral GNNs and existing $K$-hop GNNs with the graph diffusion kernel like MixHop~\citep{abu2019mixhop} and MAGNA~\citep{wang2021multihop}, we find they actually use a \textbf{weak version of $K$-hop than the definition of us}. Specifically, it is shown that the expressive power of spectral GNNs is also bounded by 1-WL test~\cite{wang2022powerful}, which contradicts our result as graph diffusion can be viewed as a special case of spectral GNN. However, we show that our definition of $K$-hop message passing with graph diffusion kernel actually injects a non-linear function on the spectral basis, thus achieving superior expressive power. We leave the detailed discussion in Appendix~\ref{app:khop_compare}. Further, Distance Encoding~\citep{li2020distance} also uses the shortest path distance information to augment the 1-hop message passing, which is similar to $K$-hop GNNs with the shortest path distance kernel. However, we find the expressive power of the two frameworks differs from each other. We leave the detailed discussion in Appendix~\ref{app:khop_de}.

\subsection{Limitation of $K$-hop message passing framework}
Although we show that $K$-hop GNNs with $K>1$ are better at distinguishing non-isomorphic structures than 1-hop GNNs, there are still limitations. In this section, we discuss the limitation of $K$-hop message passing. Specifically, we show that \textbf{the choice of the kernel can affect the expressive power of $K$-hop message passing}. Furthermore, even with $K$-hop message passing, we still cannot distinguish some simple non-isomorphic structures and \textbf{the expressive power of $K$-hop message passing is bounded by 3-WL}.

Continue looking at the provided examples in Figure~\ref{fig:regular1}. In example 1, if we use the shortest path distance kernel instead of the graph diffusion kernel, two nodes have the same number of neighbors in the 2nd hop, which means that we cannot distinguish two nodes this time. Similarly, in example 2, two nodes have the same number of neighbors in both 1st and 2nd hops using graph diffusion kernel. These results highlight that the choice of the kernel can affect the expressive power of $K$-hop message passing, and none of them can distinguish both two examples with 2-hop message passing. 

Recently,~\citet{Frasca2022UnderstandingAE} show that any subgraph-based GNNs with node-based selection policy can be implemented by 3-IGN~\citep{maron2019provably,maron2018invariant} and thus their expressive power is bounded by 3-WL test. Here, we show that the $K$-hop message passing GNNs can also be implemented by 3-IGN for both two kernels and thus:
\begin{theorem}
\label{thm:khop_limitation}
The expressive power of a proper $K$-hop message passing GNN of any kernel is bounded by the 3-WL test.
\end{theorem}
We include the proof in Appendix~\ref{app:khop_limitation}. Given all these observations, we may wonder if there is a way to further improve the expressive power of $K$-hop message passing?

\section{KP-GNN: improving the power of $K$-hop message passing by peripheral subgraph}
In this section, we describe how to improve the expressive power of $K$-hop message passing by adding additional information to the message passing framework. Specifically, by adding \textit{peripheral subgraph} information, we can improve the expressive power of the $K$-hop message passing by a large margin. 
\subsection{Peripheral edge and peripheral subgraph}
First, we define \textit{peripheral edge} and \textit{peripheral subgraph}.
\begin{definition}
\textit{The peripheral edge $E(Q^{k,t}_{v,G})$ is defined as the set of edges that connect nodes within set $Q^{k,t}_{v,G}$. We further denote $|E(Q^{k,t}_{v,G})|$ as the number of peripheral edge in $E(Q^{k,t}_{v,G})$. The peripheral subgraph $G^{k,t}_{v,G}=(Q^{k,t}_{v,G},E(Q^{k,t}_{v,G}))$ is defined as the subgraph induced by $Q^{k,t}_{v,G}$ from the whole graph $G$}.
\end{definition}
Briefly speaking, the peripheral edge $E(Q^{k,t}_{v,G})$ record all the edges whose two ends are both from $Q^{k,t}_{v,G}$ and the peripheral subgraph is a graph constituted by peripheral edges. It is easy to see that the peripheral subgraph $G^{k,t}_{v,G}$ automatically contains all the information of peripheral edge $E(Q^{k,t}_{v,G})$. Next, we show that the power of $K$-hop message passing can be improved by leveraging the information of peripheral edges and peripheral subgraphs. We again refer to the examples in Figure~\ref{fig:regular1}. Here we only consider the peripheral edge information. In example 1, we notice that at the 1st hop, there is an edge between node $3$ and node $4$ in the left graph. More specifically, $E(Q^{1,t}_{v_1,G^{(1)}})=\{(3,4)\}$. In contrast, we have $E(Q^{1,t}_{v_2,G^{(2)}})=\{\}$ in the right graph, which means there is no edge between the 1st hop neighbors of $v_2$. Therefore, we can successfully distinguish these two nodes by adding this information to the message passing. Similarly, in example 2, there is one edge between the 1st hop neighbors of node $v_2$, but no such edge exists for node $v_1$. By leveraging peripheral edge information, we can also distinguish the two nodes. The above examples demonstrate the effectiveness of the peripheral edge and peripheral subgraph information.
\subsection{$K$-hop peripheral-subgraph-enhanced graph neural network}
In this section, we propose \textbf{K}-hop \textbf{P}eripheral-subgraph-enhanced  \textbf{G}raph \textbf{N}eural \textbf{N}etwork (KP-GNN), which equips $K$-hop message passing GNNs with peripheral subgraph information for more powerful GNN design. Recall the $K$-hop message passing defined in Equation~(\ref{eq:khop_mp}). The only difference between KP-GNN and original $K$-hop GNNs is that we revise the message function as follows:
\begin{align}
\label{eq:KP-GNN}
m^{l,k}_{v}=\ &\text{MES}^{l}_{k}(\{\!\!\{(h^{l-1}_{u},~e_{uv})|u\in Q^{k,t}_{v,G}\}\!\!\}, ~G^{k,t}_{v,G}).
\end{align}
Briefly speaking, in the message step at the $k$-th hop, we not only aggregate information of the neighbors but also the peripheral subgraph at that hop. The implementation of KP-GNN can be very flexible, as any graph encoding function can be used.
To maximize the information the model can encode while keeping it simple, we implement the message function as:

\begin{equation}
\begin{split}
\label{eq:KP-GNN_imp}
\text{MES}^{l}_{k}=\ \text{MES}^{l,normal}_{k}&(\{\!\!\{(h^{l-1}_{u},e_{uv})|u\in Q^{k,t}_{v,G}\}\!\!\})+f(G^{k,t}_{v,G}),\\
f(G^{k,t}_{v,G})=\ \text{EMB}&((E(Q^{k,t}_{v,G}),C^{k^{\prime}}_k))~,
\end{split}
\end{equation}

where $\text{MES}^{l,normal}_{k}$ denotes the message function in the original GNN model, $C^{k^{\prime}}_k$ is the $k^{\prime}$-configuration, which encode both node configuration and the number of the peripheral edge of all nodes in $G^{k,t}_{v,G}$ up to $k^{\prime}$ hops. It can be regarded as running another 1 layer KP-GNN and readout function on each peripheral subgraph. $\text{EMB}$ is a learnable embedding function. With this implementation, any base GNN model can be incorporated into and be enhanced by the KP-GNN framework by replacing $\text{MES}^{l,normal}_{k}$ and $\text{UPD}^l_k$ with the corresponding functions for each hop $k$. We leave the detailed implementation in Appendix~\ref{app:implementation}.

\subsection{The expressive power of KP-GNN and comparison with existing methods}
In this section, we theoretically characterize the expressive power of KP-GNN and compare it with the original $K$-hop message passing framework. The key insight is that, according to Equation~(\ref{eq:KP-GNN}), the message function at the $k$-th hop additionally encodes $G^{k,t}_{v,G}$ compared to normal $K$-hop message passing. As we have already shown in the last section, $K$-hop GNNs are bounded by 3-WL and thus cannot distinguish any non-isomorphic distance regular graphs, as well as Distance Encoding~\cite{li2020distance}. Let $C^{k^{\prime}}_{j,G}$ be the $k^{\prime}$-configuration of peripheral subgraph at $j$-th hop of nodes in distance regular graph $G$. Here we show that with the aid of peripheral subgraphs, KP-GNN is able to distinguish distance regular graphs:
\begin{proposition}
\label{pro:kp_drg}
\textit{For two non-isomorphic distance regular graphs $G^{(1)}=(V^{(1)},E^{(1)})$ and $G^{(2)}=(V^{(2)},E^{(2)})$ with the same diameter $d$ and intersection array $(b_0,b_1,...,b_{d-1};c_1,c_2,...,c_d)$. Given a proper 1-layer $d$-hop KP-GNN with message functions defined in Equation~(\ref{eq:KP-GNN_imp}), it can distinguish $G^{(1)}$ and $G^{(2)}$ if $C^{k^{\prime}}_{j,G^{(1)}} \neq C^{k^{\prime}}_{j,G^{(2)}}$ for some $0<j \leq d$.}
\end{proposition}
We include the proof in Appendix~\ref{app:kp_expressive}. Here we leverage an example in Figure~\ref{fig:DRG} to briefly show why KP-GNN is able to distinguish distance regular graphs. Figure~\ref{fig:DRG} displays two distance regular graphs with an intersection array of $(6,3;1,2)$. The left one is Shrikhande graph and the right one is 4$\times$4 Rook's graph. Now, let's look at the 1-hop peripheral subgraph of the green node. In the Shrikhande graph, there are 6 peripheral edges marked with red. Further, 6 edges constitute a circle. In the 4$\times$4 Rook's graph, there are still 6 peripheral edges. However, 6 edges constitute two circles with 3 edges in each circle, which is different from the Shrikhande graph. Then, any peripheral subgraph encoder that can distinguish these two graphs like node configuration enables the corresponding KP-GNN to distinguish the example. Proposition~\ref{pro:kp_drg} shows that the KP-GNN is capable of distinguishing distance regular graphs, which further distinguishes KP-GNN from DE-1~\citep{li2020distance} as it cannot distinguish any two connected distance regular graphs with the same intersection arrays according to Theorem 3.7 in~\cite{li2020distance}. However, it is currently unknown whether can KP-GNN with Equation~\ref{eq:KP-GNN_imp} distinguish all distance regular graphs. We leave the detailed discussion in Appendix~\ref{app:kp_expressive}.

\begin{figure*}[t]
\centering
\vspace{-10pt}
\includegraphics[width=0.8\textwidth]{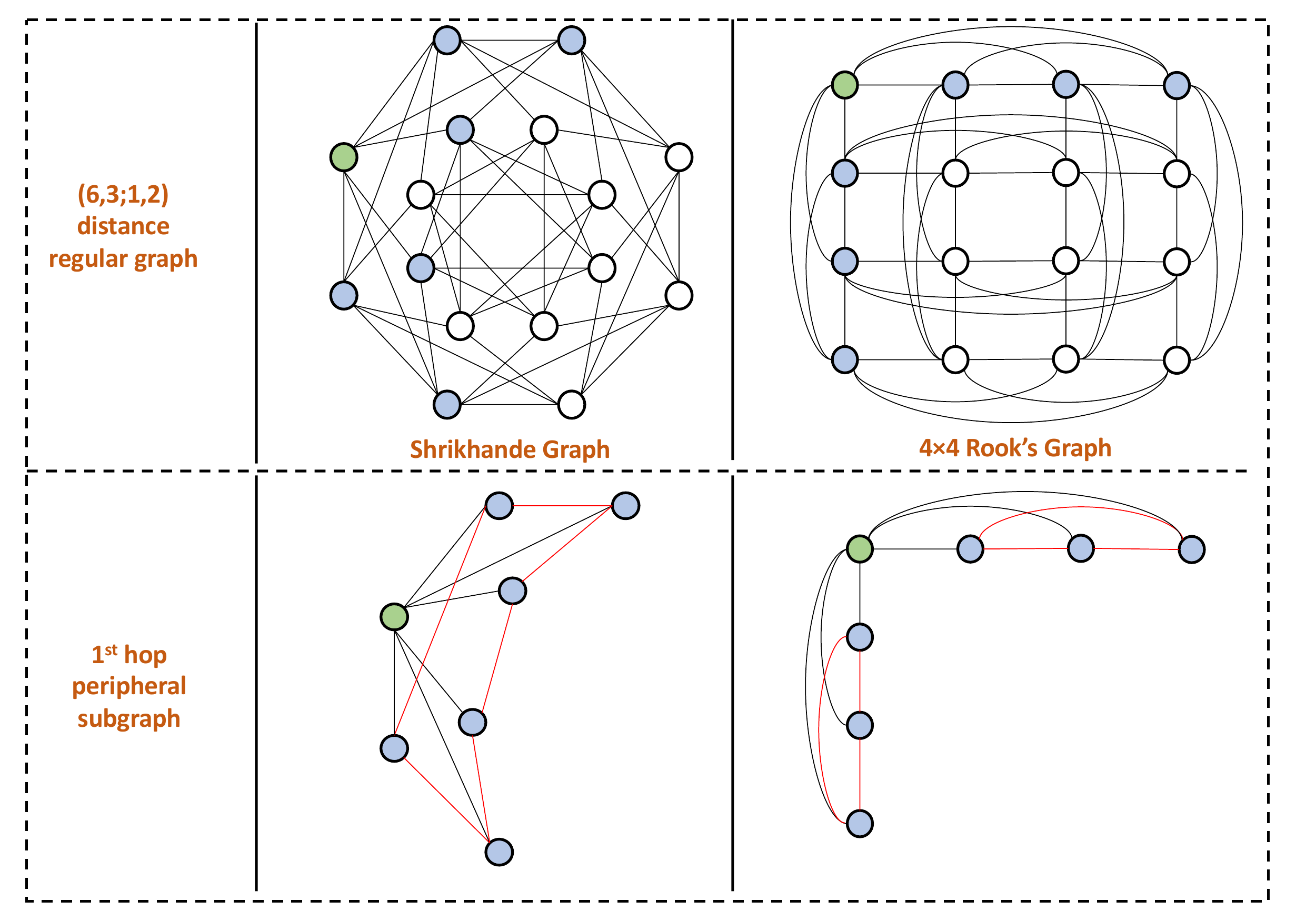}
\vspace{-5pt}
\caption{An example of two non-isomorphic distance regular graph with intersection array $(6,3;1,2)$.}
 \vspace{-15pt}
\label{fig:DRG}
\end{figure*}

Moreover, both the subgraph-based GNNs like NGNN~\citep{zhang2021nested}, GNN-AK~\citep{zhao2022from}, ESAN~\citep{bevilacqua2022equivariant}, and KP-GNN leverage the information in the subgraph to enhance the power of message passing. However, KP-GNN is intrinsically different from them. Firstly, in KP-GNN, the message passing is performed on the whole graph instead of the subgraphs. This means that for each node, there is only one representation to be learned. Instead, for subgraph-based GNNs, the message passing is performed separately for each subgraph and each node could have multiple representations depending on which subgraph it is in. Secondly, in subgraph-based GNNs, they consider the subgraph as a whole without distinguishing nodes at different hops. Instead, KP-GNN takes one step further by dividing the subgraph into two parts. The first part is the hierarchy of neighbors at each hop. The second part is the connection structure between nodes in each hop. This gives us a better point of view to design a more powerful learning method. From the Corollary 7 in~\cite{Frasca2022UnderstandingAE}, we know that all subgraph-based GNNs with node selection as subgraph policy is bounded by 3-WL, which means they cannot distinguish any distance regular graph and KP-GNN is better at it.

\subsection{Time, space complexity, and limitation}
In this section, we discuss the time and space complexity of $K$-hop message passing GNN and KP-GNN. Suppose a graph has $n$ nodes and $m$ edges. Then, the $K$-hop message passing and KP-GNN have both the space complexity of $O(n)$ and the time complexity of $O(n^2)$ for the shortest path distance kernel. Note that the complexity of graph diffusion is no less than the shortest path distance kernel. \textbf{We can see that KP-GNN only requires the same space complexity as vanilla GNNs and much less time complexity than the subgraph-based GNNs, which are at least $O(nm)$}. However, $K$-hop message passing including KP-GNN still have intrinsic limitation. We leave a detailed discussion on the complexity and limitation of KP-GNN in Appendix~\ref{app:complexity}.

\section{Related Work}
\textbf{Expressive power of GNN.} Analyzing the expressive power of GNNs is a crucial problem as it can serve as a guide on how to improve GNNs. \citet{xu2018powerful} and \citet{morris2019weisfeiler} first proved that the power of 1-hop message passing is bounded by the 1-WL test. In other words, 1-hop message passing cannot distinguish any non-isomorphic graphs that the 1-WL test fails to. In recent years, many efforts have been put into increasing the expressive power of 1-hop messaging passing. The first line of research tries to mimic the higher-order WL tests, like 1-2-3 GNN~\citep{morris2019weisfeiler}, PPGN~\citep{maron2019provably}, ring-GNN~\citep{chen2019equivalence}. However, they require exponentially increasing space and time complexity w.r.t. node number and cannot be generalized to large-scale graphs. The second line of research tries to enhance the rooted subtree of 1-WL with additional features. Some works~\citep{sato2020random,abboud2020surprising,loukas2019graph} add one-hot or random features into nodes. Although they achieve good results in some settings, they deteriorate the generalization ability as such features produce different representations for nodes even with the same local graph structure. Some works like Distance Encoding~\citep{li2020distance}, SEAL~\citep{zhang2018link}, labeling trick~\citep{zhang2021labeling} and GLASS~\citep{wang2022glass} introduce node labeling based on either distance or distinguishing target node set. On the other hand, GraphSNN~\citep{wijesinghe2022a} introduces a hierarchy of local isomorphism and proposes structural coefficients as additional features to identify such local isomorphism. However, the function designed to approximate the structural coefficient cannot fully achieve its theoretical power. The third line of research resorts to subgraph representation. Specifically, ID-GNN~\citep{you2021identity} extracts ego-netwok for each node and labels the root node with a different color. NGNN~\citep{zhang2021nested} encodes a rooted subgraph instead of a rooted subtree by subgraph pooling thus achieving superior expressive power on distinguishing regular graphs. GNN-AK~\citep{zhao2022from} applies a similar idea as NGNN. The only difference lies in how to compute the node representation from the local subgraph. However, such methods need to run an inner GNN on every node of the graph thus introducing much more computation overhead. Meanwhile, the expressive power of subgraph GNNs are bounded by 3-WL~\cite{Frasca2022UnderstandingAE}.

\textbf{$K$-hop message passing GNN.}
There are some existing works that instantiate the $K$-hop message passing framework. For example, MixHop~\citep{abu2019mixhop} performs message passing on each hop with graph diffusion kernel and concatenates the representation on each hop as the final representation. K-hop~\citep{nikolentzos2020k} sequentially performs the message passing from hop K to hop 1 to compute the representation of the center node. However, it is not parallelizable due to its computational procedure. MAGNA~\citep{wang2021multihop} introduces an attention mechanism to $K$-hop message passing. GPR-GNN~\citep{chien2021adaptive} use graph diffusion kernel to perform graph convolution on $K$-hop and aggregate them with learnable parameters. However, none of them give a formal definition of $K$-hop message passing and theoretically analyze its representation power and limitations.  

\section{Experiments}
In this section, we conduct extensive experiments to evaluate the performance of KP-GNN. Specifically, we 1) empirically verify the expressive power of KP-GNN on 3 simulation datasets and demonstrate the benefits of KP-GNN compared to normal $K$-hop message passing GNNs; 2) demonstrate the effectiveness of KP-GNN on identifying various node properties, graph properties, and substructures with 3 simulation datasets; 3) show that the KP-GNN can achieve state-of-the-art performance on multiple real-world datasets; 4) analyze the running time of KP-GNN. The detail of each variant of KP-GNN is described in Appendix~\ref{app:implementation} and the detailed experimental setting is described in Appendix~\ref{app:experiential}. We implement the KP-GNN with PyTorch Geometric package~\cite{fey2019fast}. Our code is available at \url{https://github.com/JiaruiFeng/KP-GNN}.

\textbf{Datasets}: To evaluate the expressive power of KP-GNN, we choose: 1) EXP dataset~\citep{abboud2020surprising}, which contains 600 pairs of non-isomorphic graphs (1-WL failed). The goal is to map these graphs to two different classes. 2) SR25 dataset~\citep{balcilar2021breaking}, which contains 15 non-isomorphic strongly regular graphs (3-WL failed) with each graph of 25 nodes. The dataset is translated to a 15-way classification problem with the goal of mapping each graph into different classes. 3) CSL dataset~\citep{murphy2019relational}, which contains 150 4-regular graphs (1-WL failed) divided into 10 isomorphism classes. The goal of the task is to classify them into corresponding isomorphism classes. To demonstrate the capacity of KP-GNN on counting node/graph properties and substructures, we pick 1) Graph property regression (connectedness, diameter, radius) and node property regression (single source shortest path, eccentricity, Laplacian feature) task on random graph dataset~\citep{corso2020principal}. 2) Graph substructure counting (triangle, tailed triangle, star, and 4-cycle) tasks on random graph dataset~\citep{chen2020can}. To evaluate the performance of KP-GNN on real-world datasets, we select 1) MUTAG~\citep{debnath1991structure}, D\&D~\citep{dobson2003distinguishing}, PROTEINS~\citep{dobson2003distinguishing}, PTC-MR~\citep{toivonen2003statistical}, and IMDB-B~\citep{yanardag2015deep} from TU database. 2) QM9~\citep{ramakrishnan2014quantum,wu2018moleculenet} and ZINC~\citep{dwivedi2020benchmarkgnns} for molecular properties prediction. The detailed statistics of the datasets are described in Appendix~\ref{app:datasets}. Without further highlighting, all error bars in the result tables are the standard deviations of multiple runs.

\begin{wraptable}[12]{L}{0.54\textwidth}
%\setlength{\tabcolsep}{3.7pt}
% \vspace{-10pt}
\vspace{-15pt}
\begin{center}

\caption{Empirical evaluation of the expressive power.}
\hspace{-17pt}
\resizebox{0.5\textwidth}{!}{
\begin{minipage}[t]{1.1\linewidth}
  %\fontsize{10}{8}\selectfont
   \vspace{-5pt}
   \centering
\label{tab:expressive}
\begin{tabular}{l|l|cc|cc|cc}
\toprule
\multirow{2}{*}{Method}                   & \multirow{2}{*}{K} & \multicolumn{2}{l}{EXP (ACC)} & \multicolumn{2}{l}{SR (ACC)} & \multicolumn{2}{l}{CSL (ACC)} \\
    \cmidrule(l{2pt}r{2pt}){3-4}\cmidrule(l{2pt}r{2pt}){5-6} \cmidrule(l{2pt}r{2pt}){7-8} 
                                          &                    & SPD           & GD            & SPD           & GD           & SPD           & GD            \\
    \midrule
\multirow{4}{*}{\textbf{K-GIN}}           & K=1                & 50            & 50            & 6.67          & 6.67         & 12            & 12            \\
                                          & K=2                & 50            & 50            & 6.67          & 6.67         & 32            & 22.7          \\
                                          & K=3                & 100           & 66.9          & 6.67          & 6.67         & 62            & 42            \\
                                          & K=4                & 100           & 100           & 6.67          & 6.67         & 92.7          & 62.7          \\
\midrule
\multirow{4}{*}{\textbf{KP-GIN}} & K=1                & 50            & 50            & 100           & 100          & 22            & 22            \\
                                          & K=2                & 100           & 100           & 100           & 100          & 52.7          & 52.7          \\
                                          & K=3                & 100           & 100           & 100           & 100          & 90            & 90            \\
                                          & K=4                & 100           & 100           & 100           & 100          & 100           & 100 \\
    \bottomrule
\end{tabular}
\end{minipage}
}
\end{center}
\vspace{-8pt}
\hspace{-5pt}
\end{wraptable}

\textbf{Empirical evaluation of the expressive power}: 
For empirical evaluation of the expressive power, we conduct the ablation study on hop $K$ for both normal $K$-hop GNNs and KP-GNN. For $K$-hop GNNs, we implement K-GIN which uses GIN~\citep{xu2018powerful} as the base encoder. For KP-GNN, we implement KP-GIN. The results are shown in Table~\ref{tab:expressive}. Based on the results, we have the following conclusions: 1) $K$-hop GNNs with both two kernels have expressive power higher than the 1-WL test as it shows the perfect performance on the EXP dataset and performance better than a random guess on the CSL dataset. 2) Increasing $K$ can improve the expressive for both two kernels. 3) $K$-hop GNNs cannot distinguish any strong regular graphs in SR25 dataset, which is aligned with Theorem~\ref{thm:khop_limitation}. 4) KP-GNN has much higher expressive power than normal $K$-hop GNNs by showing better performance on every dataset given the same $K$. Further, it achieves perfect results on the SR25 dataset even with $K=1$, which demonstrates its ability on distinguishing distance regular graphs.

\begin{table*}[h]
    \vspace{-8pt}
\caption{Simulation dataset result. The top two are highlighted by \first{First}, \second{Second}. }
   \vspace{-4pt}
  \resizebox{1.0\textwidth}{!}{
  \begin{minipage}[t]{1.28\linewidth}

    %\vspace{-5pt}
    \centering
    \label{tab:simulation}
  \begin{tabular}{l|ccc|ccc|cccc}
   \toprule
     \multirow{2}{*}{Method} & 
     \multicolumn{3}{c}{Node Properties ($\log_{10}$(MSE))}&  \multicolumn{3}{c}{Graph Properties ($\log_{10}$(MSE))}&
     \multicolumn{4}{c}{Counting Substructures (MAE)} \\
    \cmidrule(l{2pt}r{2pt}){2-4}\cmidrule(l{2pt}r{2pt}){5-7} \cmidrule(l{2pt}r{2pt}){8-11} 
    &  SSSP & Ecc. & Lap. &  Connect. & Diameter & Radius &
    Tri. & Tailed Tri. & Star & 4-Cycle \\ 
    \midrule  
    \textbf{GIN} & -2.0000 & -1.9000 & -1.6000 & -1.9239 & -3.3079 & -4.7584 & 0.3569 & 0.2373 & 0.0224 & 0.2185  \\
    \midrule  
    \textbf{PNA} &\first{-2.8900} & \first{-2.8900}& -3.7700 &  -1.9395 & 3.4382 & -4.9470 & 0.3532& 0.2648 & 0.1278 & 0.2430 \\
    \textbf{PPGN} &- &- &- & -1.9804 & -3.6147 & -5.0878 & \second{0.0089} & \second{0.0096} & \first{0.0148} & \first{0.0090} \\
    \textbf{GIN-AK+} & -& -& -&  \second{-2.7513} & \second{-3.9687} & -5.1846 & 0.0123 &  0.0112 & \second{0.0150} & \second{0.0126} \\
    \midrule
    \textbf{K-GIN+} & -2.7919 & -2.5938& \second{-4.6360} &  -2.1782 & \first{-3.9695} & \second{-5.3088} & 0.2593 & 0.1930 & 0.0165  & 0.2079 \\
    \textbf{KP-GIN+}  & \second{-2.7969} & \second{-2.6169}& \first{-4.7687}& \first{-4.4322}  & -3.9361 & \first{-5.3345} & \first{0.0060} & \first{0.0073} & 0.0151 & 0.0395\\
  \bottomrule
\end{tabular}
\end{minipage}
}
\vspace{-10pt}
\end{table*}

\textbf{Effectiveness on node/graph properties and substructure prediction}: To evaluate the effectiveness of KP-GNN on node/graph properties and substructure prediction, we compare it with several existing models. For the baseline model, we use GIN~\citep{xu2018powerful}, which has the same expressive power as the 1-WL test. For more powerful baselines, we use GIN-AK+~\citep{zhao2022from}, PNA~\citep{corso2020principal}, and PPGN~\citep{maron2019provably}. For normal $K$-hop GNNs, we implement K-GIN+, and for KP-GNN, we implement KP-GIN+. The results are shown in Table~\ref{tab:simulation}. Baseline results are taken from~\cite{zhao2022from} and~\cite{corso2020principal}. We can see KP-GIN+ achieve SOTA on a majority of tasks. Meanwhile, K-GIN+ also gets great performance on node/graph properties prediction. These results demonstrate the capability of KP-GNN to identify various properties and substructures. We leave the detailed results on counting substructures in Appendix~\ref{app:ablation}

\begin{table*}[h]
%\setlength{\tabcolsep}{3.7pt}
% \vspace{-10pt}
\vspace{-5pt}
\begin{center}
%\hspace{-5pt}

\caption{TU dataset evaluation result.}
\resizebox{0.9\textwidth}{!}{
\begin{minipage}[t]{0.95\linewidth}
  %\fontsize{10}{8}\selectfont
   \vspace{-5pt}
\label{tab:tu}
  \begin{tabular}{lccccc}
   \toprule
    Method & MUTAG & D\&D & PTC-MR & PROTEINS & IMDB-B\\
    \midrule  
    \textbf{WL} & 90.4{\small$\pm$5.7} &  79.4{\small$\pm$0.3} & 59.9{\small$\pm$4.3} &75.0{\small$\pm$3.1} & 73.8{\small$\pm$3.9}  \\
    \midrule
    \textbf{GIN} & 89.4{\small$\pm$5.6} & - & 64.6{\small$\pm$7.0} & 75.9{\small$\pm$2.8} & 75.1{\small$\pm$5.1}\\
    \textbf{DGCNN} & 85.8{\small$\pm$1.7} & 79.3 {\small$\pm$0.9} & 58.6 {\small$\pm$2.5} & 75.5{\small$\pm$0.9} & 70.0{\small$\pm$0.9} \\
    \midrule
    \textbf{GraphSNN} & 91.24{\small$\pm$2.5} & 82.46{\small$\pm$2.7} & 66.96{\small$\pm$3.5} & 76.51{\small$\pm$2.5} & 76.93{\small$\pm$3.3}\\
    \textbf{GIN-AK+} & 91.30{\small$\pm$7.0} & - & 68.20{\small$\pm$5.6} & 77.10{\small$\pm$5.7} & 75.60{\small$\pm$3.7}\\
    \midrule
    \textbf{KP-GCN} & 91.7{\small$\pm$6.0} & 79.0{\small$\pm$4.7} & 67.1{\small$\pm$6.3} & 75.8{\small$\pm$3.5} & 75.9{\small$\pm$3.8}\\
    \textbf{KP-GraphSAGE} & 91.7{\small$\pm$6.5} & 78.1{\small$\pm$2.6} & 66.5{\small$\pm$4.0} & 76.5{\small$\pm$4.6} & 76.4{\small$\pm$2.7} \\
    \textbf{KP-GIN} & 92.2{\small$\pm$6.5} & 79.4{\small$\pm$3.8} & 66.8{\small$\pm$6.8} & 75.8{\small$\pm$4.6} &  76.6{\small$\pm$4.2}\\
    \midrule
   \textbf{GIN-AK+$^*$} & 95.0{\small$\pm$6.1} & OOM & 74.1{\small$\pm$5.9} & 78.9{\small$\pm$5.4} & 77.3{\small$\pm$3.1} \\
   \textbf{GraphSNN$^*$} & 94.70{\small$\pm$1.9} & \textbf{83.93{\small$\pm$2.3}} & 70.58{\small$\pm$3.1} & 78.42{\small$\pm$2.7} & 78.51{\small$\pm$2.8} \\ 
     \textbf{KP-GCN$^*$} & \textbf{96.1{\small$\pm$4.6}} & 83.2{\small$\pm$2.2} & \textbf{77.1{\small$\pm$4.1}} & 80.3{\small$\pm$4.2} & 79.6{\small$\pm$2.5}\\
    \textbf{KP-GraphSAGE$^*$} & \textbf{96.1{\small$\pm$4.6}} & 83.6{\small$\pm$2.4}& 76.2{\small$\pm$4.5} & \textbf{80.4{\small$\pm$4.3}} & 80.3{\small$\pm$2.4} \\
    \textbf{KP-GIN$^*$} & 95.6{\small$\pm$4.4} & 83.5{\small$\pm$2.2} & 76.2{\small$\pm$4.5} & 79.5{\small$\pm$4.4} & \textbf{80.7{\small$\pm$2.6}} \\
 
  \bottomrule
\end{tabular}
%\vspace{-10pt}
\end{minipage}
}
\end{center}
\vspace{-10pt}
\end{table*}

\textbf{Evaluation on TU datasets}: For baseline models, we select: 1) graph kernel-based method: WL subtree kernel~\citep{shervashidze2011weisfeiler}; 2) vanilla GNN methods: GIN~\citep{xu2018powerful} and DGCNN~\citep{zhang2018end}; 3) advanced GNN methods: GraphSNN~\citep{wijesinghe2022a} and GIN-AK+~\citep{zhao2022from}. For the proposed KP-GNN, we implement GCN~\citep{kipf2017semisupervised}, GraphSAGE~\citep{hamilton2017inductive}, and GIN~\citep{xu2018powerful} using the KP-GNN framework, denoted as KP-GCN, KP-GraphSAGE, and KP-GIN respectively. The results are shown in Table~\ref{tab:tu}. For a more fair and comprehensive comparison, we report the results from two different evaluation settings. The first setting follows~\citet{xu2018powerful} and the second setting follows~\citet{wijesinghe2022a}. We denote the second setting with $^*$ in the table. We can see KP-GNN achieves SOTA performance on most of datasets under the second setting and still comparable performance to other baselines under the first setting.

\begin{table}[t]
\vspace{-15pt}
\centering
   \caption{QM9 results. The top two are highlighted by \first{First}, \second{Second}.}
%\setlength{\tabcolsep}{3.7pt}
%\begin{center}
\hspace{-25pt}
  \resizebox{0.9\textwidth}{!}{
  \begin{minipage}[t]{0.95\textwidth}

   %\vspace{-5pt}
   \label{tab:qm9}
  \begin{tabular}{l|ccccc|cc}
    \toprule
    \textbf{Target} &\textbf{DTNN}&\textbf{MPNN}&\textbf{Deep LRP} &\textbf{PPGN} &\textbf{N-1-2-3-GNN}& \textbf{KP-GIN+} & \textbf{KP-GIN$'$}\\
    \midrule
    $\mu$ & \second{0.244} &  0.358 & 0.364 & \first{0.231}& 0.433 & 0.367 & 0.358 \\
    $\alpha$ & 0.95 &  0.89 & 0.298 & 0.382 & 0.265 & \second{0.242} & \first{0.233} \\
    $\varepsilon_{\text{HOMO}}$ & 0.00388 &  0.00541 & 0.00254 & 0.00276 & 0.00279 & \second{0.00247} & \first{0.00240} \\
    $\varepsilon_{\text{LUMO}}$ & 0.00512 &  0.00623 & 0.00277&0.00287 & 0.00276 & \second{0.00238} & \first{0.00236}\\
    $\Delta \varepsilon$ & 0.0112 &  0.0066 & 0.00353 &0.00406 & 0.00390 & \second{0.00345}  & \first{0.00333}\\
    $\langle R^2 \rangle$ & 17.0 &  28.5 & 19.3 & 16.7 & 20.1 & \first{16.49} & \second{16.51} \\
    ZPVE & 0.00172 &  0.00216 & 0.00055 & 0.00064 &\first{0.00015} & 0.00018 & \second{0.00017} \\
    $U_0$ & 2.43 &  2.05 & 0.413 & 0.234 & 0.205 & \second{0.0728} & \first{0.0682}\\
    $U$ & 2.43 &  2.00 & 0.413 & 0.234 & 0.200 & \first{0.0553} & \second{0.0696}\\
    $H$ & 2.43 & 2.02 & 0.413 & 0.229 & 0.249 & \first{0.0575} & \second{0.0641}\\
    $G$ & 2.43 & 2.02 & 0.413 & 0.238 & 0.253 &\second{0.0526} & \first{0.0484} \\
    $C_v$ & 0.27 &  0.42 & 0.129 & 0.184 & \first{0.0811} & 0.0973 & \second{0.0869} \\ 
  \bottomrule
\end{tabular}
\end{minipage}
}
%\end{center}
\vspace{-18pt}
\end{table}

\begin{wraptable}[11]{L}{0.42\textwidth}
\vspace{-14pt}
\caption{ZINC result.}
   \vspace{-5pt}
\resizebox{0.47\textwidth}{!}{
\begin{minipage}[t]{0.5\textwidth}
  %\fontsize{10}{8}\selectfont
\label{tab:zinc}
\begin{tabular}{l|c|c}
\toprule
Method & \# param. & test MAE\\
\midrule
\textbf{MPNN} & 480805 & 0.145{\small$\pm$0.007}\\
\textbf{PNA} & 387155 & 0.142{\small$\pm$0.010}\\
\textbf{Graphormer} & 489321 & 0.122{\small$\pm$0.006}\\
\textbf{GSN} & \textasciitilde 500000 & 0.101{\small$\pm$0.010}\\
\textbf{GIN-AK+} & - & 0.080{\small$\pm$0.001}\\
\textbf{CIN} & - & \textbf{0.079{\small$\pm$0.006}}\\
\midrule
\textbf{KP-GIN+} & 499099 & 0.111{\small$\pm$0.006} \\
\textbf{KP-GIN$'$} & 488649 & 0.093{\small$\pm$0.007} \\
  \bottomrule

\end{tabular}
\end{minipage}
}
\end{wraptable}

\textbf{Evaluation on molecular prediction tasks}:  For QM9 dataset, we report baseline results of DTNN and MPNN from~\citep{wu2018moleculenet}. We further select Deep LRP~\citep{chen2020can}, PPGN~\citep{maron2019provably}, and Nested 1-2-3-GNN~\citep{zhang2021nested} as baseline models. For the ZINC dataset, we report results of MPNN~\citep{gilmer2017neural} and PNA~\citep{corso2020principal} from~\cite{ying2021do}. We further pick Graphormer~\citep{ying2021do}, GSN~\citep{bouritsas2021improving}, GIN-AK+~\citep{zhao2022from}, and CIN~\citep{bodnar2021weisfeiler}. For KP-GNN, we choose KP-GIN+ and KP-GIN$'$. The results of the QM9 dataset are shown in Table~\ref{tab:qm9}. We can see KP-GNN achieves SOTA performance on most of the targets. The results of the ZINC dataset are shown in Table~\ref{tab:zinc}. Although KP-GNN does not achieve the best result, it is still comparable to other methods.

\begin{wraptable}[8]{L}{0.48\textwidth}
\vspace{-10pt}
\caption{Running time (s/epoch).}
   \vspace{-5pt}
\resizebox{0.5\textwidth}{!}{
\begin{minipage}[t]{0.5\textwidth}
  %\fontsize{10}{8}\selectfont
\label{tab:running_time}
  \begin{tabular}{llcc}
   \toprule
     Method & D\&D  & ZINC & Graph property \\
    \midrule  
    \textbf{GIN} & 1.10 & 3.59 & 1.02 \\
    \textbf{K-GIN} & 3.94 & 6.44 & 1.67 \\
    \textbf{KP-GIN} & 4.19 & 7.38 & 1.94 \\
    \textbf{KP-GIN+} & 4.28 & 6.74 & 1.93 \\
  \bottomrule
\hspace{-5pt}
\end{tabular}
\end{minipage}
}
\end{wraptable}

\textbf{Running time comparison}: In this section, we compare the running time of KP-GNN to 1-hop message passing GNN and $K$-hop message passing GNN. We use GIN~\citep{xu2018powerful} as the base model. We also include the KP-GIN+. All models use the same number of layers and hidden dimensions for a fair comparison. The results are shown in Table~\ref{tab:running_time}. We set $K=4$ for all datasets. We can see the computational overhead is almost linear to $K$. This is reasonable as practical graphs are sparse and the number of $K$-hop neighbors is far less than $n$ when using a small $K$. 

\section{Conclusion}
In this paper, we theoretically characterize the power of $K$-hop message passing GNNs and propose the KP-GNN to improve the expressive power by leveraging the peripheral subgraph information at each hop. Theoretically, we prove that $K$-hop GNNs can distinguish almost all regular graphs but are bounded by the 3-WL test. KP-GNN is able to distinguish many distance regular graphs. Empirically, KP-GNN achieves competitive results across all simulation and real-world datasets.  

\section{Acknowledgement}
This work is partially supported by NSF grant CBE-2225809 and NSF China (No. 62276003).

\medskip

\bibliography{references}

\begin{thebibliography}{58}
\providecommand{\natexlab}[1]{#1}
\providecommand{\url}[1]{\texttt{#1}}
\expandafter\ifx\csname urlstyle\endcsname\relax
  \providecommand{\doi}[1]{doi: #1}\else
  \providecommand{\doi}{doi: \begingroup \urlstyle{rm}\Url}\fi

\bibitem[Kipf and Welling(2017)]{kipf2017semisupervised}
Thomas~N. Kipf and Max Welling.
\newblock Semi-supervised classification with graph convolutional networks.
\newblock In \emph{International Conference on Learning Representations}, 2017.

\bibitem[Duvenaud et~al.(2015)Duvenaud, Maclaurin, Iparraguirre, Bombarell,
  Hirzel, Aspuru-Guzik, and Adams]{duvenaud2015convolutional}
David~K Duvenaud, Dougal Maclaurin, Jorge Iparraguirre, Rafael Bombarell,
  Timothy Hirzel, Al{\'a}n Aspuru-Guzik, and Ryan~P Adams.
\newblock Convolutional networks on graphs for learning molecular fingerprints.
\newblock In \emph{Advances in neural information processing systems}, pages
  2224--2232, 2015.

\bibitem[Hamilton et~al.(2017)Hamilton, Ying, and
  Leskovec]{hamilton2017inductive}
Will Hamilton, Zhitao Ying, and Jure Leskovec.
\newblock Inductive representation learning on large graphs.
\newblock In \emph{Advances in Neural Information Processing Systems}, pages
  1025--1035, 2017.

\bibitem[Veličković et~al.(2018)Veličković, Cucurull, Casanova, Romero,
  Liò, and Bengio]{veličković2018graph}
Petar Veličković, Guillem Cucurull, Arantxa Casanova, Adriana Romero, Pietro
  Liò, and Yoshua Bengio.
\newblock Graph attention networks.
\newblock In \emph{International Conference on Learning Representations}, 2018.
\newblock URL \url{https://openreview.net/forum?id=rJXMpikCZ}.

\bibitem[Li et~al.(2015)Li, Tarlow, Brockschmidt, and Zemel]{li2015gated}
Yujia Li, Daniel Tarlow, Marc Brockschmidt, and Richard Zemel.
\newblock Gated graph sequence neural networks.
\newblock \emph{arXiv preprint arXiv:1511.05493}, 2015.

\bibitem[Zhang et~al.(2018)Zhang, Cui, Neumann, and Chen]{zhang2018end}
Muhan Zhang, Zhicheng Cui, Marion Neumann, and Yixin Chen.
\newblock An end-to-end deep learning architecture for graph classification.
\newblock In \emph{AAAI}, pages 4438--4445, 2018.

\bibitem[Xu et~al.(2019)Xu, Hu, Leskovec, and Jegelka]{xu2018powerful}
Keyulu Xu, Weihua Hu, Jure Leskovec, and Stefanie Jegelka.
\newblock How powerful are graph neural networks?
\newblock In \emph{International Conference on Learning Representations}, 2019.
\newblock URL \url{https://openreview.net/forum?id=ryGs6iA5Km}.

\bibitem[Grover and Leskovec(2016)]{grover2016node2vec}
Aditya Grover and Jure Leskovec.
\newblock node2vec: Scalable feature learning for networks.
\newblock In \emph{Proceedings of the 22nd ACM SIGKDD international conference
  on Knowledge discovery and data mining}, pages 855--864. ACM, 2016.

\bibitem[Perozzi et~al.(2014)Perozzi, Al-Rfou, and Skiena]{perozzi2014deepwalk}
Bryan Perozzi, Rami Al-Rfou, and Steven Skiena.
\newblock Deepwalk: Online learning of social representations.
\newblock In \emph{Proceedings of the 20th ACM SIGKDD international conference
  on Knowledge discovery and data mining}, pages 701--710. ACM, 2014.

\bibitem[Weisfeiler and Lehman(1968)]{weisfeiler1968reduction}
Boris Weisfeiler and AA~Lehman.
\newblock A reduction of a graph to a canonical form and an algebra arising
  during this reduction.
\newblock \emph{Nauchno-Technicheskaya Informatsia}, 2\penalty0 (9):\penalty0
  12--16, 1968.

\bibitem[Morris et~al.(2019)Morris, Ritzert, Fey, Hamilton, Lenssen, Rattan,
  and Grohe]{morris2019weisfeiler}
Christopher Morris, Martin Ritzert, Matthias Fey, William~L Hamilton, Jan~Eric
  Lenssen, Gaurav Rattan, and Martin Grohe.
\newblock Weisfeiler and leman go neural: Higher-order graph neural networks.
\newblock In \emph{Proceedings of the AAAI Conference on Artificial
  Intelligence}, volume~33, pages 4602--4609, 2019.

\bibitem[Abu-El-Haija et~al.(2019)Abu-El-Haija, Perozzi, Kapoor, Alipourfard,
  Lerman, Harutyunyan, Ver~Steeg, and Galstyan]{abu2019mixhop}
Sami Abu-El-Haija, Bryan Perozzi, Amol Kapoor, Nazanin Alipourfard, Kristina
  Lerman, Hrayr Harutyunyan, Greg Ver~Steeg, and Aram Galstyan.
\newblock Mixhop: Higher-order graph convolutional architectures via sparsified
  neighborhood mixing.
\newblock In \emph{international conference on machine learning}, pages 21--29.
  PMLR, 2019.

\bibitem[Nikolentzos et~al.(2020)Nikolentzos, Dasoulas, and
  Vazirgiannis]{nikolentzos2020k}
Giannis Nikolentzos, George Dasoulas, and Michalis Vazirgiannis.
\newblock k-hop graph neural networks.
\newblock \emph{Neural Networks}, 130:\penalty0 195--205, 2020.

\bibitem[Wang et~al.(2021)Wang, Ying, Huang, and Leskovec]{wang2021multihop}
Guangtao Wang, Rex Ying, Jing Huang, and Jure Leskovec.
\newblock Multi-hop attention graph neural networks.
\newblock In \emph{International Joint Conference on Artificial Intelligence},
  2021.

\bibitem[Chien et~al.(2021)Chien, Peng, Li, and Milenkovic]{chien2021adaptive}
Eli Chien, Jianhao Peng, Pan Li, and Olgica Milenkovic.
\newblock Adaptive universal generalized pagerank graph neural network.
\newblock In \emph{International Conference on Learning Representations}, 2021.
\newblock URL \url{https://openreview.net/forum?id=n6jl7fLxrP}.

\bibitem[Brossard et~al.(2020)Brossard, Frigo, and Dehaene]{brossard2020graph}
R{\'e}my Brossard, Oriel Frigo, and David Dehaene.
\newblock Graph convolutions that can finally model local structure.
\newblock \emph{arXiv preprint arXiv:2011.15069}, 2020.

\bibitem[Ying et~al.(2021)Ying, Cai, Luo, Zheng, Ke, He, Shen, and
  Liu]{ying2021do}
Chengxuan Ying, Tianle Cai, Shengjie Luo, Shuxin Zheng, Guolin Ke, Di~He,
  Yanming Shen, and Tie-Yan Liu.
\newblock Do transformers really perform badly for graph representation?
\newblock In A.~Beygelzimer, Y.~Dauphin, P.~Liang, and J.~Wortman Vaughan,
  editors, \emph{Advances in Neural Information Processing Systems}, 2021.
\newblock URL \url{https://openreview.net/forum?id=OeWooOxFwDa}.

\bibitem[Gilmer et~al.(2017)Gilmer, Schoenholz, Riley, Vinyals, and
  Dahl]{gilmer2017neural}
Justin Gilmer, Samuel~S Schoenholz, Patrick~F Riley, Oriol Vinyals, and
  George~E Dahl.
\newblock Neural message passing for quantum chemistry.
\newblock In \emph{Proceedings of the 34th International Conference on Machine
  Learning-Volume 70}, pages 1263--1272. JMLR. org, 2017.

\bibitem[Cybenko(1989)]{Cybenko1989ApproximationBS}
George~V. Cybenko.
\newblock Approximation by superpositions of a sigmoidal function.
\newblock \emph{Mathematics of Control, Signals and Systems}, 2:\penalty0
  303--314, 1989.

\bibitem[Zaheer et~al.(2017)Zaheer, Kottur, Ravanbakhsh, Poczos, Salakhutdinov,
  and Smola]{zaheer2017deep}
Manzil Zaheer, Satwik Kottur, Siamak Ravanbakhsh, Barnabas Poczos, Russ~R
  Salakhutdinov, and Alexander~J Smola.
\newblock Deep sets.
\newblock In \emph{Advances in Neural Information Processing Systems}, pages
  3391--3401, 2017.

\bibitem[Wang and Zhang(2022{\natexlab{a}})]{wang2022powerful}
Xiyuan Wang and Muhan Zhang.
\newblock How powerful are spectral graph neural networks.
\newblock \emph{ICML}, 2022{\natexlab{a}}.

\bibitem[Li et~al.(2020)Li, Wang, Wang, and Leskovec]{li2020distance}
Pan Li, Yanbang Wang, Hongwei Wang, and Jure Leskovec.
\newblock Distance encoding: Design provably more powerful neural networks for
  graph representation learning.
\newblock In \emph{Advances in Neural Information Processing Systems},
  volume~33, pages 4465--4478, 2020.

\bibitem[Frasca et~al.(2022)Frasca, Bevilacqua, Bronstein, and
  Maron]{Frasca2022UnderstandingAE}
Fabrizio Frasca, Beatrice Bevilacqua, Michael~M Bronstein, and Haggai Maron.
\newblock Understanding and extending subgraph gnns by rethinking their
  symmetries.
\newblock In \emph{Advances in Neural Information Processing Systems}, 2022.

\bibitem[Maron et~al.(2019{\natexlab{a}})Maron, Ben-Hamu, Serviansky, and
  Lipman]{maron2019provably}
Haggai Maron, Heli Ben-Hamu, Hadar Serviansky, and Yaron Lipman.
\newblock Provably powerful graph networks.
\newblock In \emph{Advances in Neural Information Processing Systems}, pages
  2156--2167, 2019{\natexlab{a}}.

\bibitem[Maron et~al.(2019{\natexlab{b}})Maron, Ben-Hamu, Shamir, and
  Lipman]{maron2018invariant}
Haggai Maron, Heli Ben-Hamu, Nadav Shamir, and Yaron Lipman.
\newblock Invariant and equivariant graph networks.
\newblock In \emph{International Conference on Learning Representations},
  2019{\natexlab{b}}.
\newblock URL \url{https://openreview.net/forum?id=Syx72jC9tm}.

\bibitem[Zhang and Li(2021)]{zhang2021nested}
Muhan Zhang and Pan Li.
\newblock Nested graph neural networks.
\newblock In A.~Beygelzimer, Y.~Dauphin, P.~Liang, and J.~Wortman Vaughan,
  editors, \emph{Advances in Neural Information Processing Systems}, 2021.
\newblock URL \url{https://openreview.net/forum?id=7_eLEvFjCi3}.

\bibitem[Zhao et~al.(2022)Zhao, Jin, Akoglu, and Shah]{zhao2022from}
Lingxiao Zhao, Wei Jin, Leman Akoglu, and Neil Shah.
\newblock From stars to subgraphs: Uplifting any {GNN} with local structure
  awareness.
\newblock In \emph{International Conference on Learning Representations}, 2022.
\newblock URL \url{https://openreview.net/forum?id=Mspk_WYKoEH}.

\bibitem[Bevilacqua et~al.(2022)Bevilacqua, Frasca, Lim, Srinivasan, Cai,
  Balamurugan, Bronstein, and Maron]{bevilacqua2022equivariant}
Beatrice Bevilacqua, Fabrizio Frasca, Derek Lim, Balasubramaniam Srinivasan,
  Chen Cai, Gopinath Balamurugan, Michael~M. Bronstein, and Haggai Maron.
\newblock Equivariant subgraph aggregation networks.
\newblock In \emph{International Conference on Learning Representations}, 2022.
\newblock URL \url{https://openreview.net/forum?id=dFbKQaRk15w}.

\bibitem[Chen et~al.(2019)Chen, Villar, Chen, and Bruna]{chen2019equivalence}
Zhengdao Chen, Soledad Villar, Lei Chen, and Joan Bruna.
\newblock On the equivalence between graph isomorphism testing and function
  approximation with gnns.
\newblock In \emph{Advances in Neural Information Processing Systems}, pages
  15894--15902, 2019.

\bibitem[Sato et~al.(2021)Sato, Yamada, and Kashima]{sato2020random}
Ryoma Sato, Makoto Yamada, and Hisashi Kashima.
\newblock Random features strengthen graph neural networks.
\newblock In \emph{Proceedings of the 2021 SIAM International Conference on
  Data Mining (SDM)}, 2021.

\bibitem[Abboud et~al.(2021)Abboud, Ceylan, Grohe, and
  Lukasiewicz]{abboud2020surprising}
Ralph Abboud, İsmail~İlkan Ceylan, Martin Grohe, and Thomas Lukasiewicz.
\newblock The surprising power of graph neural networks with random node
  initialization.
\newblock In \emph{Proceedings of the Thirtieth International Joint Conference
  on Artificial Intelligence, {IJCAI-21}}, pages 2112--2118, 8 2021.
\newblock \doi{10.24963/ijcai.2021/291}.
\newblock URL \url{https://doi.org/10.24963/ijcai.2021/291}.

\bibitem[Loukas(2019)]{loukas2019graph}
Andreas Loukas.
\newblock What graph neural networks cannot learn: depth vs width.
\newblock \emph{arXiv preprint arXiv:1907.03199}, 2019.

\bibitem[Zhang and Chen(2018)]{zhang2018link}
Muhan Zhang and Yixin Chen.
\newblock Link prediction based on graph neural networks.
\newblock In \emph{Advances in Neural Information Processing Systems}, pages
  5165--5175, 2018.

\bibitem[Zhang et~al.(2021)Zhang, Li, Xia, Wang, and Jin]{zhang2021labeling}
Muhan Zhang, Pan Li, Yinglong Xia, Kai Wang, and Long Jin.
\newblock Labeling trick: A theory of using graph neural networks for
  multi-node representation learning.
\newblock In A.~Beygelzimer, Y.~Dauphin, P.~Liang, and J.~Wortman Vaughan,
  editors, \emph{Advances in Neural Information Processing Systems}, 2021.
\newblock URL \url{https://openreview.net/forum?id=Hcr9mgBG6ds}.

\bibitem[Wang and Zhang(2022{\natexlab{b}})]{wang2022glass}
Xiyuan Wang and Muhan Zhang.
\newblock {GLASS}: {GNN} with labeling tricks for subgraph representation
  learning.
\newblock In \emph{International Conference on Learning Representations},
  2022{\natexlab{b}}.
\newblock URL \url{https://openreview.net/forum?id=XLxhEjKNbXj}.

\bibitem[Wijesinghe and Wang(2022)]{wijesinghe2022a}
Asiri Wijesinghe and Qing Wang.
\newblock A new perspective on ''how graph neural networks go beyond
  weisfeiler-lehman?''.
\newblock In \emph{International Conference on Learning Representations}, 2022.
\newblock URL \url{https://openreview.net/forum?id=uxgg9o7bI_3}.

\bibitem[You et~al.(2021)You, Gomes-Selman, Ying, and
  Leskovec]{you2021identity}
Jiaxuan You, Jonathan~M Gomes-Selman, Rex Ying, and Jure Leskovec.
\newblock Identity-aware graph neural networks.
\newblock \emph{Proceedings of the AAAI Conference on Artificial Intelligence},
  35\penalty0 (12):\penalty0 10737--10745, May 2021.

\bibitem[Fey and Lenssen(2019)]{fey2019fast}
Matthias Fey and Jan~Eric Lenssen.
\newblock Fast graph representation learning with pytorch geometric.
\newblock \emph{arXiv preprint arXiv:1903.02428}, 2019.

\bibitem[Balcilar et~al.(2021)Balcilar, H{\'e}roux, Ga{\"u}z{\`e}re, Vasseur,
  Adam, and Honeine]{balcilar2021breaking}
Muhammet Balcilar, Pierre H{\'e}roux, Benoit Ga{\"u}z{\`e}re, Pascal Vasseur,
  S{\'e}bastien Adam, and Paul Honeine.
\newblock Breaking the limits of message passing graph neural networks.
\newblock In \emph{Proceedings of the 38th International Conference on Machine
  Learning (ICML)}, 2021.

\bibitem[Murphy et~al.(2019)Murphy, Srinivasan, Rao, and
  Ribeiro]{murphy2019relational}
Ryan Murphy, Balasubramaniam Srinivasan, Vinayak Rao, and Bruno Ribeiro.
\newblock Relational pooling for graph representations.
\newblock In \emph{International Conference on Machine Learning}, pages
  4663--4673. PMLR, 2019.

\bibitem[Corso et~al.(2020)Corso, Cavalleri, Beaini, Li\`{o}, and
  Veli\v{c}kovi\'{c}]{corso2020principal}
Gabriele Corso, Luca Cavalleri, Dominique Beaini, Pietro Li\`{o}, and Petar
  Veli\v{c}kovi\'{c}.
\newblock Principal neighbourhood aggregation for graph nets.
\newblock In \emph{Advances in Neural Information Processing Systems}, 2020.

\bibitem[Chen et~al.(2020)Chen, Chen, Villar, and Bruna]{chen2020can}
Zhengdao Chen, Lei Chen, Soledad Villar, and Joan Bruna.
\newblock Can graph neural networks count substructures?
\newblock \emph{Advances in neural information processing systems}, 2020.

\bibitem[Debnath et~al.(1991)Debnath, Lopez, Debnath, Shusterman, and
  Hansch]{debnath1991structure}
Asim~Kumar Debnath, de~Compadre~RL Lopez, Gargi Debnath, Alan~J Shusterman, and
  Corwin Hansch.
\newblock Structure-activity relationship of mutagenic aromatic and
  heteroaromatic nitro compounds. correlation with molecular orbital energies
  and hydrophobicity.
\newblock \emph{Journal of medicinal chemistry}, 34\penalty0 (2):\penalty0
  786--797, 1991.

\bibitem[Dobson and Doig(2003)]{dobson2003distinguishing}
Paul~D Dobson and Andrew~J Doig.
\newblock Distinguishing enzyme structures from non-enzymes without alignments.
\newblock \emph{Journal of molecular biology}, 330\penalty0 (4):\penalty0
  771--783, 2003.

\bibitem[Toivonen et~al.(2003)Toivonen, Srinivasan, King, Kramer, and
  Helma]{toivonen2003statistical}
Hannu Toivonen, Ashwin Srinivasan, Ross~D King, Stefan Kramer, and Christoph
  Helma.
\newblock Statistical evaluation of the predictive toxicology challenge
  2000--2001.
\newblock \emph{Bioinformatics}, 19\penalty0 (10):\penalty0 1183--1193, 2003.

\bibitem[Yanardag and Vishwanathan(2015)]{yanardag2015deep}
Pinar Yanardag and SVN Vishwanathan.
\newblock Deep graph kernels.
\newblock In \emph{Proceedings of the 21th ACM SIGKDD International Conference
  on Knowledge Discovery and Data Mining}, pages 1365--1374. ACM, 2015.

\bibitem[Ramakrishnan et~al.(2014)Ramakrishnan, Dral, Rupp, and
  Von~Lilienfeld]{ramakrishnan2014quantum}
Raghunathan Ramakrishnan, Pavlo~O Dral, Matthias Rupp, and O~Anatole
  Von~Lilienfeld.
\newblock Quantum chemistry structures and properties of 134 kilo molecules.
\newblock \emph{Scientific data}, 1\penalty0 (1):\penalty0 1--7, 2014.

\bibitem[Wu et~al.(2018)Wu, Ramsundar, Feinberg, Gomes, Geniesse, Pappu,
  Leswing, and Pande]{wu2018moleculenet}
Zhenqin Wu, Bharath Ramsundar, Evan~N Feinberg, Joseph Gomes, Caleb Geniesse,
  Aneesh~S Pappu, Karl Leswing, and Vijay Pande.
\newblock Moleculenet: a benchmark for molecular machine learning.
\newblock \emph{Chemical science}, 9\penalty0 (2):\penalty0 513--530, 2018.

\bibitem[Dwivedi et~al.(2020)Dwivedi, Joshi, Luu, Laurent, Bengio, and
  Bresson]{dwivedi2020benchmarkgnns}
Vijay~Prakash Dwivedi, Chaitanya~K Joshi, Anh~Tuan Luu, Thomas Laurent, Yoshua
  Bengio, and Xavier Bresson.
\newblock Benchmarking graph neural networks.
\newblock \emph{arXiv preprint arXiv:2003.00982}, 2020.

\bibitem[Shervashidze et~al.(2011)Shervashidze, Schweitzer, Leeuwen, Mehlhorn,
  and Borgwardt]{shervashidze2011weisfeiler}
Nino Shervashidze, Pascal Schweitzer, Erik Jan~van Leeuwen, Kurt Mehlhorn, and
  Karsten~M Borgwardt.
\newblock Weisfeiler-lehman graph kernels.
\newblock \emph{Journal of Machine Learning Research}, 12\penalty0
  (Sep):\penalty0 2539--2561, 2011.

\bibitem[Bouritsas et~al.(2021)Bouritsas, Frasca, Zafeiriou, and
  Bronstein]{bouritsas2021improving}
Giorgos Bouritsas, Fabrizio Frasca, Stefanos Zafeiriou, and Michael~M.
  Bronstein.
\newblock Improving graph neural network expressivity via subgraph isomorphism
  counting, 2021.
\newblock URL \url{https://openreview.net/forum?id=LT0KSFnQDWF}.

\bibitem[Bodnar et~al.(2021)Bodnar, Frasca, Otter, Wang, Li{\`o}, Montufar, and
  Bronstein]{bodnar2021weisfeiler}
Cristian Bodnar, Fabrizio Frasca, Nina Otter, Yu~Guang Wang, Pietro Li{\`o},
  Guido Montufar, and Michael~M. Bronstein.
\newblock Weisfeiler and lehman go cellular: {CW} networks.
\newblock In A.~Beygelzimer, Y.~Dauphin, P.~Liang, and J.~Wortman Vaughan,
  editors, \emph{Advances in Neural Information Processing Systems}, 2021.
\newblock URL \url{https://openreview.net/forum?id=uVPZCMVtsSG}.

\bibitem[Bollobás(1982)]{BOLLOBAS198233}
Béla Bollobás.
\newblock Distinguishing vertices of random graphs.
\newblock In Béla Bollobás, editor, \emph{Graph Theory}, volume~62 of
  \emph{North-Holland Mathematics Studies}, pages 33--49. North-Holland, 1982.
\newblock \doi{https://doi.org/10.1016/S0304-0208(08)73545-X}.
\newblock URL
  \url{https://www.sciencedirect.com/science/article/pii/S030402080873545X}.

\bibitem[Bollobás(1980)]{bollobas_probabilistic_1980}
Béla Bollobás.
\newblock A {Probabilistic} {Proof} of an {Asymptotic} {Formula} for the
  {Number} of {Labelled} {Regular} {Graphs}.
\newblock \emph{European Journal of Combinatorics}, 1\penalty0 (4):\penalty0
  311--316, December 1980.
\newblock ISSN 0195-6698.
\newblock \doi{10.1016/S0195-6698(80)80030-8}.
\newblock URL
  \url{https://www.sciencedirect.com/science/article/pii/S0195669880800308}.

\bibitem[Vaswani et~al.(2017)Vaswani, Shazeer, Parmar, Uszkoreit, Jones, Gomez,
  Kaiser, and Polosukhin]{vaswani2017attention}
Ashish Vaswani, Noam Shazeer, Niki Parmar, Jakob Uszkoreit, Llion Jones,
  Aidan~N Gomez, {\L}ukasz Kaiser, and Illia Polosukhin.
\newblock Attention is all you need.
\newblock In \emph{Advances in neural information processing systems}, pages
  5998--6008, 2017.

\bibitem[Geerts(2020)]{Geerts2020TheEP}
Floris Geerts.
\newblock The expressive power of kth-order invariant graph networks.
\newblock \emph{ArXiv}, abs/2007.12035, 2020.

\bibitem[Luong et~al.(2015)Luong, Pham, and Manning]{luong-etal-2015-effective}
Thang Luong, Hieu Pham, and Christopher~D. Manning.
\newblock Effective approaches to attention-based neural machine translation.
\newblock In \emph{Proceedings of the 2015 Conference on Empirical Methods in
  Natural Language Processing}, pages 1412--1421, Lisbon, Portugal, September
  2015. Association for Computational Linguistics.
\newblock \doi{10.18653/v1/D15-1166}.
\newblock URL \url{https://aclanthology.org/D15-1166}.

\bibitem[Xu et~al.(2018)Xu, Li, Tian, Sonobe, ichi Kawarabayashi, and
  Jegelka]{xu2018representation}
Keyulu Xu, Chengtao Li, Yonglong Tian, Tomohiro Sonobe, Ken ichi Kawarabayashi,
  and Stefanie Jegelka.
\newblock Representation learning on graphs with jumping knowledge networks.
\newblock In \emph{International Conference on Machine Learning}, 2018.

\end{thebibliography}
\bibliographystyle{unsrtnat}

\newpage
\appendix

\section{More about the $K$-hop kernel and $K$-hop message passing}
\label{app:khop_more}
\begin{figure*}[h]
\centering
% \vspace{5pt}
\includegraphics[width=0.95\textwidth]{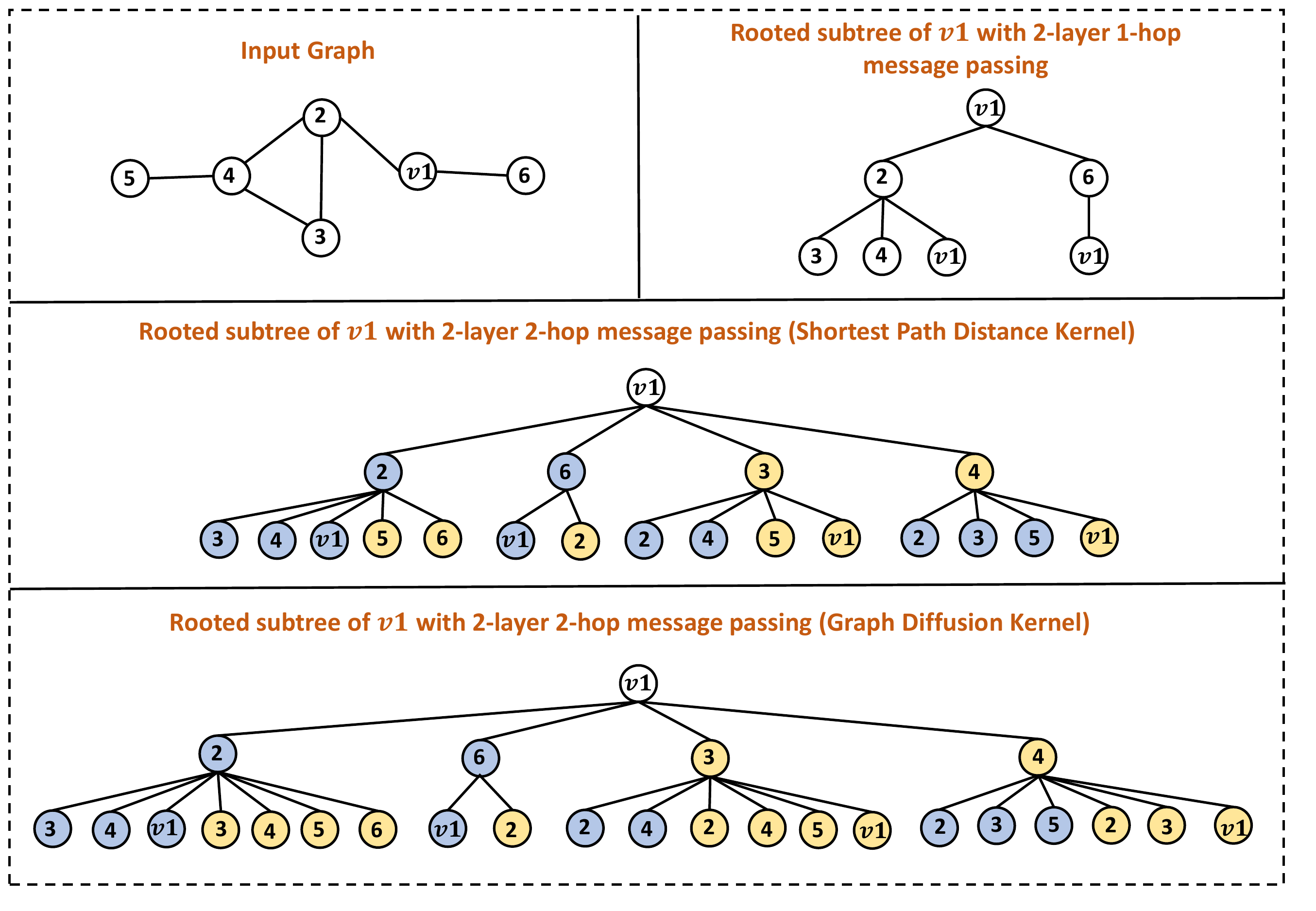}
\vspace{-5pt}
\caption{The rooted subtree of node $v1$ with 1-hop message passing and $K$-hop message passing. Here we assume that $K=2$ and the number of layers is 2. }
 \vspace{-15pt}
\label{fig:khopMP}
\end{figure*}
In this section, we further discuss two different types of K-hop kernel and K-hop message passing.
\subsection{More about $K$-hop kernel}
 First, recall the shortest path distance kernel and graph diffusion kernel defined in Definition~\ref{def:spd} and~\ref{def:gd}. Given two definitions, the first thing we can conclude is that the $K$-hop neighbors of node $v$ under two different kernels will be the same, namely $\mathcal{N}^{K,spd}_{v,G}=\mathcal{N}^{K,gd}_{v,G}$ as both two kernels capture all nodes that can be reached from node $v$ within the distance of $K$. Second, we have $\mathcal{N}^{1,spd}_{v,G}=Q^{1,spd}_{v,G}=\mathcal{N}^{1,gd}_{v,G}=Q^{1,gd}_{v,G}$, which means the neighbor set is same for both the shortest path distance kernel and the graph diffusion kernel when $K=1$. The third thing is that $Q^{k,spd}_{v,G}$ will not always equal to $Q^{k,gd}_{v,G}$ for some $k$. Since for shortest path distance kernel, one node will only appear in at most one of $Q^{k,spd}_{v,G}$ for $k=1,2,...,K$. Instead, nodes can appear in multiple $Q^{k,gd}_{v,G}$. This is the key reason why the choice of the kernel can affect the expressive power of $K$-hop message passing. 

\subsection{More about $K$-hop message passing}
Here, we use an example shown in Figure~\ref{fig:khopMP} to illustrate how K-hop message passing works and compare it with 1-hop message passing. The input graph is shown on the left top of the figure. Suppose we want to learn the representation of node $v1$ using 2 layer message passing GNNs. First, if we perform 1-hop message passing, it will encode a 2-height rooted subtree, which is shown on the right top of the figure. Note that each node is learned using the same set of parameters, which is indicated by filling each node with the same color (white in the figure). Now, we consider performing a 2-hop message passing GNN with the shortest path distance kernel. The rooted subtree of node $v1$ is shown in the middle of the figure. we can see that at each height, both 1st hop neighbors and 2nd hop neighbors are included. Furthermore, different sets of parameters are used for different hops, which is indicated by filling nodes in the different hops with different colors (blue for 1st hop and yellow for 2nd hop). Finally, at the bottom of the figure, we show the 2-hop message passing GNN with graph diffusion kernel. It is easy to see the rooted subtree is different from the one that uses the shortest path distance kernel, as nodes can appear in both the 1st hop and 2nd hop of neighbors. 

\section{Proof of injectiveness of Equation (\ref{eq:khop_mp})}\label{app:injective}

In this section, we formally prove that Equation (\ref{eq:khop_mp}) is an injective mapping of the neighbor representations at different hops. As here each layer $l$ is doing exactly the same procedure, we only need to prove it for one iteration. Therefore we ignore the superscript and rewrite Equation (\ref{eq:khop_mp}) as:

\begin{align}
\begin{aligned}
m^{k}_{v}=\text{MES}_{k}(\{\!\!\{(h_{u},e_{uv})|u\in Q^{k,t}_{v,G}&\}\!\!\}),~~
h^{k}_{v}=\text{UPD}_{k}(m^{k}_{v},h_{v}),\\
\hat{h}_{v}=\text{COMBINE}(\{\!\!\{h^{k}_{v}&|k=1,2,...,K\}\!\!\}).
\end{aligned}
\end{align}

Next, we state the following proposition:

\begin{proposition}
There exist injective functions $\text{MES}_k,\text{UPD}_k,k=1,2,...,K$, and an injective multiset function $\text{COMBINE}$, such that $\hat{h}_{v}$ is an injective mapping of $\big\{\!\!\big\{ (k, \{\!\!\{(h_{u},e_{uv})|u\in Q^{k,t}_{v,G}\}\!\!\}, h_v)  ~|~ k=1,2,...,K \big\}\!\!\big\}$.
\end{proposition}
\begin{proof}
The existence of injective message passing ($\text{MES}_k$,$\text{UPD}_k$) and multiset pooling ($\text{COMBINE}$) functions are well proved in~\citep{xu2018powerful}. So below we prove the injectiveness of $\hat{h}_v$. First, we combine $\text{MSE}_{k}$ and $\text{UPD}_{k}$ together into $\phi_k$:

\begin{align}
\begin{aligned}
h^{k}_{v}=\phi_k(\{\!\!\{(h_{u},e_{uv})|u\in Q^{k,t}_{v,G}\}\!\!\},h_v).
\end{aligned}
\end{align}

Note that $\phi_k$ is still injective as the composition of injective functions is injective. Next, we need to prove that $h^{k}_{v}$ is an injective mapping of $(k, \{\!\!\{(h_{u},e_{uv})|u\in Q^{k,t}_{v,G}\}\!\!\}, h_v)$. To prove it, we rewrite the function $\phi_k(\cdot)$ into $\phi(k)(\cdot)$, that is, $\phi$ is an injective function taking $k$ as input and outputs a function $\phi_k=\phi(k)$. We let $\phi(k_1)(x_1)$ and $\phi(k_2)(x_2)$ output different values for $k_1\neq k_2$ given any input $x_1,x_2$, e.g., always let the final output dimension be $k$. Then, we can rewrite $h^{k}_{v}$ as

\begin{align}
\begin{aligned}
h^{k}_{v}&=\phi(k)(\{\!\!\{(h_{u},e_{uv})|u \in Q^{k,t}_{v,G}\}\!\!\},h_v)\\
&=\psi(k, \{\!\!\{(h_{u},e_{uv})|u\in Q^{k,t}_{v,G}\}\!\!\}, h_v),~
\end{aligned}
\end{align}

where we have composed two injective functions $\phi(\cdot)$ and $\phi(k)(\cdot)$ into a single one $\psi(\cdot)$. Since given different $k$, $\phi(k)(\cdot)$ always outputs distinct values for any input, and given fixed $k$, $\phi(k)(\cdot)$ always outputs different values for different $(\{\!\!\{(h_{u},e_{uv})|u\in Q^{k,t}_{v,G}\}\!\!\}, h_v)$, the resulting $\psi(\cdot)$ always outputs different values for different $(k, \{\!\!\{(h_{u},e_{uv})|u\in Q^{k,t}_{v,G}\}\!\!\}, h_v)$, i.e., it injectively maps $(k, \{\!\!\{(h_{u},e_{uv})|u\in Q^{k,t}_{v,G}\}\!\!\}, h_v)$. Thus, we have proved that $h^{k}_{v}$ is an injective mapping of $(k, \{\!\!\{(h_{u},e_{uv})|u\in Q^{k,t}_{v,G}\}\!\!\}, h_v)$.

Finally, since COMBINE is an injective multiset function, we conclude that its output $\hat{h}_{v}$ is an injective mapping of $\big\{\!\!\big\{ (k, \{\!\!\{(h_{u},e_{uv})|u\in Q^{k,t}_{v,G}\}\!\!\}, h_v)  ~|~ k=1,2,...,K \big\}\!\!\big\}$.
\end{proof}
\section{Proof of Theorem~\ref{thm:khop_regular} and simulation result}
\label{app:khop_regular}
\subsection{Proof of Theorem~\ref{thm:khop_regular}}
Here we restate Theorem~\ref{thm:khop_regular}: Consider all pairs of $n$-sized $r$-regular graphs, let $ 3 \leq r < (2log2n)^{1/2}$ and $\epsilon$ be a fixed constant. With at most $K=\lfloor(\frac{1}{2}+\epsilon)\frac{\log{2n}}{\log{(r-1)}} \rfloor$, there exists a 1 layer $K$-hop message passing GNN using the shortest path distance kernel that distinguishes almost all $1-o(n^{-1/2})$ such pairs of graphs.

As we state in the main paper, 1 layer $K$-hop GNN is equivalent to inject node configuration for each node label. Therefore, it is sufficient to show that given node configuration $A^{K,spd}_{v,G}$ for all $v \in V$, we can distinguish almost every pair of regular graphs. First, we introduce the following lemma.
\begin{lemma}
\label{lem:node_configuration}
For  two graphs $G^{(1)}=(V^{(1)},E^{(1)})$ and $G^{(2)}=(V^{(2)},E^{(2)})$ that are randomly and independently sampled from $n$-sized $r$-regular graphs with $ 3 \leq r < (2log2n)^{1/2}$. We pick two nodes $v_1$ and $v_2$ from two graphs respectively.  Let $K=\lfloor(\frac{1}{2}+\epsilon)\frac{\log{2n}}{\log{(r-1)}} \rfloor$ where $\epsilon$ is a fixed constant, $A^{K,spd}_{v_1,G^{(1)}}=A^{K,spd}_{v_2,G^{(2)}}$ with the probability at most $o(n^{-3/2})$.
\end{lemma}
\begin{proof}
This Lemma can be obtained based on Theorem 6 in~\cite{BOLLOBAS198233} with minor corrections. Here we state the proof. 

We first introduce the configuration model~\cite{bollobas_probabilistic_1980} of $n$-sized $r$-regular graphs. Suppose we have $n$ disjoint sets of items, $W_i, i \in \{1,2,...,n\}$, where each set has $r$ items and corresponds to one node in the configuration model. A configuration is a partition of all $nr$ items into $\frac{nr}{2}$ pairs. Denote by $\Omega$ the set of configurations and turn it into a probability space with each configuration the same probability. Turns out among all configurations in $\Omega$, given $r< (2logn)^{1/2}$, there are about $\exp(-\frac{r^2-1}{4})$ or $\Omega ( n^{-1/2})$ portion of them are simple $r$-regular graphs~\cite{bollobas_probabilistic_1980}. Then for these configurations, If there is a pair of items with one item from set $W_i$ and another item from set $W_j$, then there is an edge between node $i$ and $j$ in the corresponding $r$-regular graph.

Let $l_0=\lfloor(\frac{1}{2}+\epsilon)\frac{\log{n}}{\log{(r-1)}} \rfloor$, We first look at two nodes randomly selected from the configuration model and consider the following procedure to generate a graph. Let node $i$ and node $j$ be the selected nodes. In the first step, we select all the edges that directly connect to node $i$ and node $j$. Then we have all nodes that are at a distance of $1$ to $i$ or $j$. In the second step, we select all the edges that connect to nodes at a distance of $1$ to either node $i$ or node $j$. Doing this iteratively for $n-1$ steps and we end up with the union of components of $i$ and $j$ in a random configuration. 

We call an edge is \textit{indispensable} if it is the first edge that ensures a node $w$ is at a certain distance to $\{i,j\}$. Similarly, an edge is \textit{dispensable} if 1) both two ends of the edge are connected to either $i$ or $j$; 2) nodes in both two ends of the edge already have at least one edge. Note that edges with both two ends in the same node are dispensable. As the first $k-1$ edges selected so far can connect to at most $k+1$ nodes, the probability that the $k$-th edge selected is dispensable is at most:

\begin{equation*}
\frac{(k+1)(r-1)}{(n-k-1)r} \approx \frac{k}{n-k}.
\end{equation*}
Therefore the probability that more than 2 of the first $k_o=\lfloor n^{1/6} \rfloor$ edges are dispensable is at most:
\begin{equation}
\label{eq:dispensable1}
{k_o \choose 3} {\left(\frac{k_o}{n-k_o}\right)}^3=o(n^{-2}).
\end{equation}
 The probability that more than $l_1=\lfloor n^{1/8}\rfloor$ of the first $k_1=\lfloor n^{6/13}\rfloor$ edges are dispensable is at most 

\begin{equation}
\label{eq:dispensable2}
{k_1 \choose l_1+1}\left( \frac{k_1}{n-k_1}\right)^{l_1+1}=o(n^{-2}).
\end{equation}

The probability that more than $l_2=\lfloor n^{5/13 }\rfloor $ of the first $k_2=\lfloor n^{2/3}\rfloor$ edges are dispensable is at most 
\begin{equation}
\label{eq:dispensable3}
{k_2 \choose l_2+1}\left( \frac{k_2}{n-k_2}\right)^{l_2+1}=o(n^{-2}).
\end{equation}

Now, let $A$ be the event that at most 2 of the first $k_o$, at most $l_1$ of the first $k_1$, and at most $l_2$ of the first $k_2$ edges be dispensable. Given Equation~(\ref{eq:dispensable1})-(\ref{eq:dispensable3}), we know the probability of $A$ is $1-o(n^{-2})$. 

\textbf{Disscussion}: Briefly speaking, event $A$ means that at the first few $k$ steps of generation, the edge will reach almost as many nodes as possible and the number of nodes at distance $\{1, 2, ..., k\}$ will be certainly $r, r(r-1), ..., r(r-1)^{k-1}$ for either $W_i$ and $W_j$. Which means $A^{k,spd}_{i,G}=A^{k,spd}_{j,G}$ with probability close to 1. Nevertheless, it also allows later members of the node configuration to be different.  

Now, we consider another way of generating a graph with a configuration model. Similar to the above, suppose we have selected all the edges connecting to at least one node with a distance less than $k$ from $\{i,j\}$ after the $k$-th step. At the $k+1$-th step, we first select all the edges with one node at a distance of $k$ from node $i$. Next, we select all the edges with both two ends at a distance of $k$ from node $j$. It is easy to see that after these two procedures, the only edges left for completing the $k+1$-th step are the edges that have one end at the distance of $k$ from the node $j$ and another end at the distance of $k+1$ from the node $j$. Suppose there are $t_k$ edges in the nodes at distance $k$ from node $j$ that have not been generated so far and there are $s_k$ nodes that have not generate any edge yet. Then the final procedure of completing the $k+1$-th step is to connect $t_k$ edges to $s_k$ nodes. This procedure goes on for every $k$ to generate the final graph.

Now let us assume that $A$ holds. It is easily seen that then for $k \leq l_0$, we have:
\begin{equation}
\label{eq:tksk}
t_k\geq (r-1)^{k-3}  \quad \textrm{and} \quad  s_k\geq n/2,
\end{equation}

where both two bounds are rather crude. Now, after the first two procedures of $k+1$-th step, we have all the nodes at the distance of $k+1$ from node $i$, which means we already determined $Q^{k+1,spd}_{i,G}$. If $|Q^{k+1,spd}_{i,G}|=|Q^{k+1,spd}_{j,G}|$, then the connection of $t_k$ edges must belong to $|Q^{k+1,spd}_{i,G}|$ nodes from totally $s_k$ nodes. The probability of $|Q^{k+1,spd}_{i,G}|=|Q^{k+1,spd}_{j,G}|$ condition on Equation~(\ref{eq:tksk}) is at most the maximum of the probability that 
$t_k$ edges connect to $l$ nodes from $s_k$ nodes with degree $r$. This probability is bounded by:

\begin{equation}
\label{eq:subset}
\underset{l}{max}\ P(|Q^{k+1,spd}_{j,G}|=l)\leq c_o \frac{s_k^{1/2}}{t_k},
\end{equation}

where we assume the $r\geq 3$ and $t_k\leq cs_k^{5/8}$ for some constant $c$ and $c_o$ is also a constant. The proof of this can be found in Lemma 7 of~\cite{bollobas_probabilistic_1980}. Given Equation~(\ref{eq:subset}), the probability that $A^{l,spd}_{i,G}=A^{l,spd}_{j,G}$ for $l \leq l_0$ is at most

\begin{equation*}
1-P(A)+\prod^{l_0}_{l=h} c_0\frac{n^{1/2}}{(r-1)^{l-3}},
\end{equation*}
where $h=\lfloor \frac{1}{2} \frac{log n}{log (r-1)}\rfloor +3$. Since $(r-1)^{l_0}\geq n^{(1+\epsilon)/2}$, the sum above is $o(n^{-2})$. Since there is at least $\Omega (n^{-1/2})$ of all the graphs generated by the configuration model that are simple $r$-regular graphs, there are at most $o(n^{-2}/n^{-1/2})=o(n^{-3/2})$ of probability that $A^{l_0,spd}_{i,G}=A^{l_0,spd}_{j,G}$.

Next, for any pair of $n$-sized $r$-regular graphs $G^{(1)}$ and $G^{(2)}$, we can combine these two graphs and generate a single regular graph with $2n$ nodes. Denote this combined graph as $G_c$ and $G_c$ has two disconnected components. It is easy to see that the above proof is still valid on $G_c$. This means that: suppose we randomly pick a node $v_1$ from the first component and node $v_2$ from another component. Then, given $3 \leq r < (2log2n)^{1/2}$ and $K=\lfloor(\frac{1}{2}+\epsilon)\frac{\log{2n}}{\log{(r-1)}} \rfloor$, we have $A^{K,spd}_{v_1,G_c} = A^{K,spd}_{v_2,G_c}$ with probability of $o(n^{-3/2})$. As node $v_1$ and $v_2$ are in two disconnected components $G^{(1)}$ and $G^{(2)}$. Therefore, it is easy to see $A^{K,spd}_{v_1,G_c}=A^{K,spd}_{v_1,G^{(1)}}$ and $A^{K,spd}_{v_2,G_c}=A^{K,spd}_{v_1,G^{(2)}}$, which completes the proof.

\end{proof}

Theorem~\ref{thm:khop_regular} is easy to prove with the aid of Lemma~\ref{lem:node_configuration}. Basically, we consider a node $v_1$ in graph $G^{(1)}$ and compare the $A^{K,spd}_{v_1,G^{(1)}}$ with $A^{K,spd}_{v_2,G^{(2)}}$ for all nodes $v_2 \in V^{(2)}$. The probability that $A^{K,spd}_{v_2,G^{(2)}} \neq A^{K,spd}_{v_1,G^{(1)}}$ for all possible $v_2$ is $1- o(n^{-3/2}n)=1-o(n^{-1/2})$. Therefore, with an injective readout function, we can guarantee that 1 layer $K$-hop message passing GNN can generate different embedding for two graphs. 

\subsection{Simulation experiments to verify Theorem~\ref{thm:khop_regular}}

\begin{figure*}[h]
\centering
\includegraphics[width=0.49\textwidth]{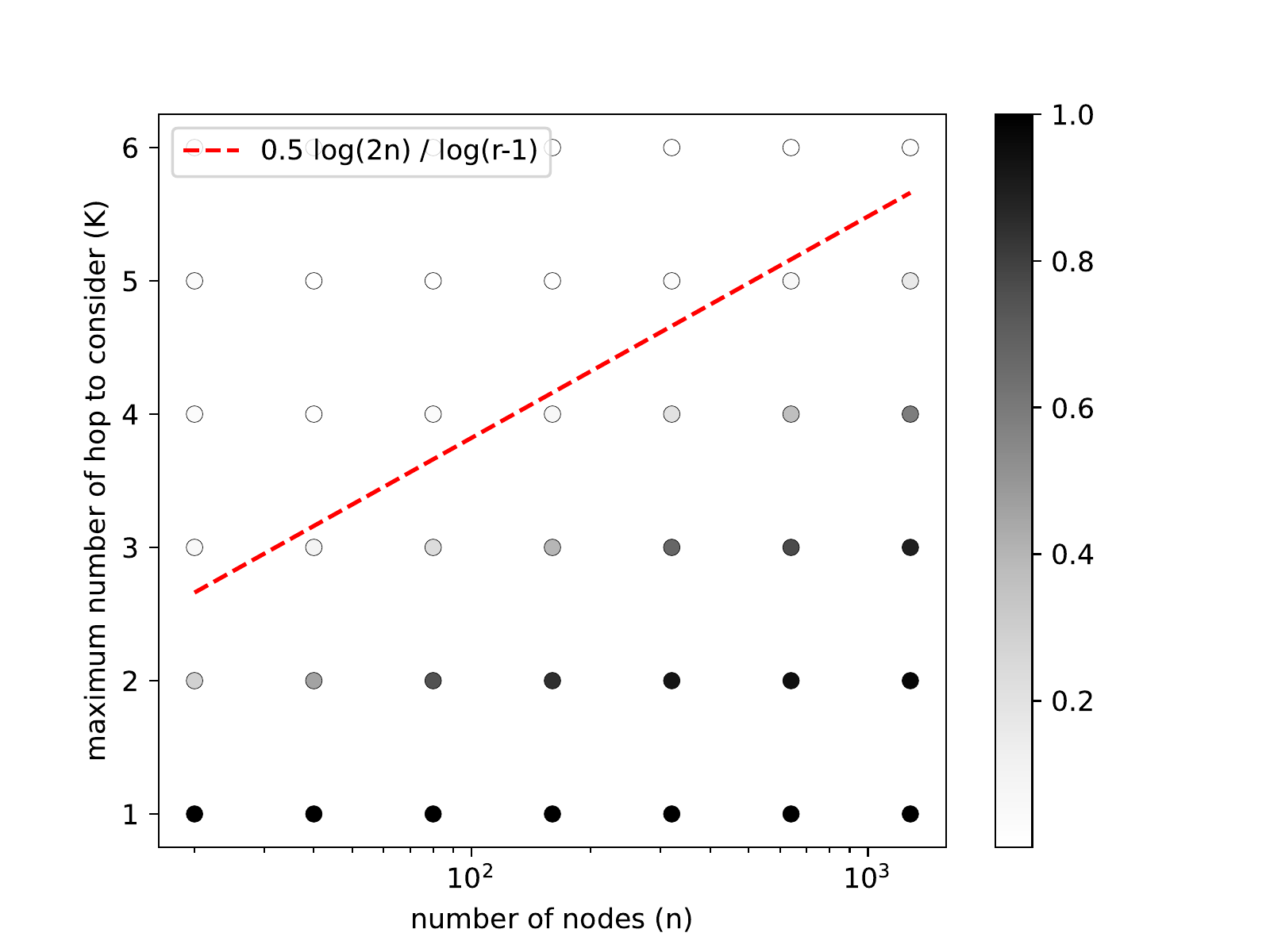}
\includegraphics[width=0.49\textwidth]{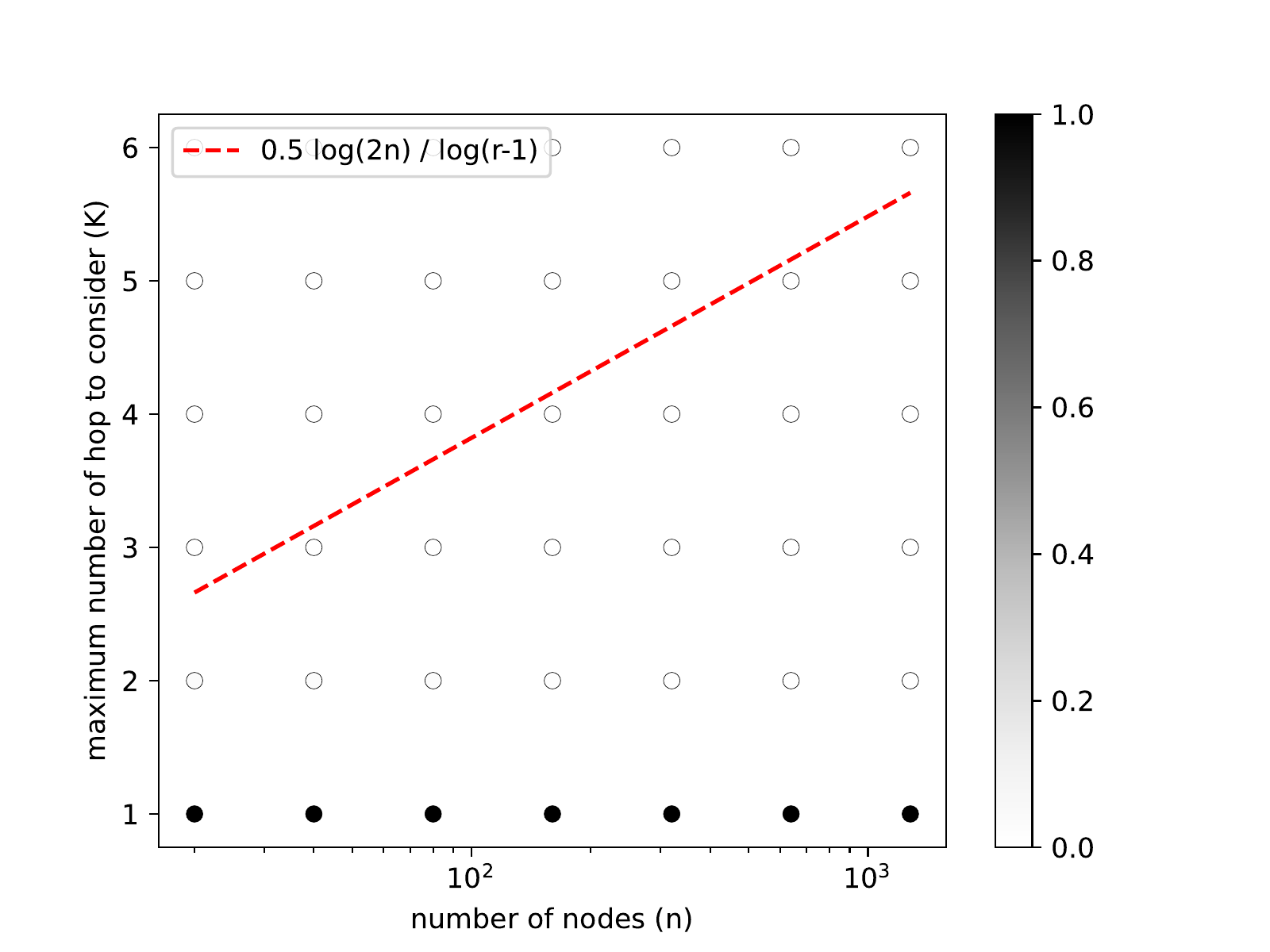}
\caption{Simulation results. The left side is the node-level result to verify Lemma~\ref{lem:node_configuration}. The right side is the graph-level result to verify Theorem~\ref{thm:khop_regular}.}
\label{fig:simulation}
\vspace{-15pt}
\end{figure*}
In this section, we conduct simulation experiments to verify both Lemma~\ref{lem:node_configuration} and Theorem~\ref{thm:khop_regular}. The results are shown in Figure~\ref{fig:simulation}. We randomly generate 100 $n$-sized 3 regular graphs with $n$ ranging from 20 to 1280. Then, we apply a 1-layer untrained K-GIN model on these graphs (1 as node feature) with $K$ range from 1 to 6. On the left side, we compare the final node representations for all nodes output by K-GIN, If the difference between two node representations $||h_v-h_u||_2$ is less than machine accuracy ($1e-10$), they are regarded as indistinguishable. The colors of the scatter plot indicate the portion of two nodes that are not distinguishable by K-GIN. The darker, the more indistinguishable node pairs. We can see that the result matches almost perfectly with Lemma~\ref{lem:node_configuration}, where $K$ is larger than $\frac{1}{2}\frac{\log{2n}}{\log{(r-1)}}$, almost all nodes are distinguishable by K-GIN. On the right side, we compare the final graph representation output by K-GIN. We can see even with $K=2$, almost all graphs are distinguishable. This is because as long as there exists one single node from one graph that has a different representation from all nodes in another graph, 1 layer K-hop message passing with an injective readout function can distinguish two graphs. 

\section{General $K$-hop color refinement and discussion on existing $K$-hop message passing GNNs}\label{app:khop_compare}
In this section, we introduce a general $K$-hop color refinement algorithm and use this algorithm to characterize the expressive power of existing $K$-hop methods further.
\subsection{General $K$-hop color refinement algorithm}
It is well-known that 1-WL test updates the label of each node in the graph by color refinement algorithm, which iteratively aggregates the label of its neighbors. Here, we extend it and define a more general color refinement algorithm. First, we denote $[n]=\{0,1,...,n\}$ and introduce \textit{refinement configuration}. 
\begin{definition}
Given $L$ and $K$, the refinement configuration $\mathbf{C}^{L,K}=(C^{0,K},C^{1,K},...,C^{L,K})$, where $C^{l,K}=(C^{l,K}_0,C^{l,K}_1,...,C^{l,K}_K)$ and $C^{l,K}_k \subseteq [l]$ for any $l \in [L] $ and $k \in [K]$.
\end{definition}
Briefly speaking, refinement configuration defines how the color refinement algorithm aggregates information from neighbors of each hops at each iteration $l$ given the maximum iteration of $L$ and the maximum number of hop $K$. Given the refinement configuration $\mathbf{C}^{L,K}$, we define general color refinement algorithm $\textbf{CR}(\mathbf{C}^{L,K})$:
\begin{align}
\begin{aligned}
\label{eq:general_cr}
R^{0}_{v,G}&=\text{LABEL}(v),\\
R^{l+1}_{v,G}=\text{HASH}((\{\!\!\{\{\!\!\{R^{s}_{u,G}|u\in Q^{0,t}_{v,G}\}\!\!\}|s \in &C^{l,K}_0\}\!\!\},...,\{\!\!\{\{\!\!\{R^{s}_{u,G}|u\in Q^{K,t}_{v,G}\}\!\!\}|s \in C^{l,K}_K\}\!\!\})),
\end{aligned}
\end{align}
where LABEL function assign the initial color to node. Equation~(\ref{eq:general_cr}) define the general color refinement algorithm. Then, given different refinement configurations, we can end up with different procedures for the algorithm. Specifically, we define the following refinement configurations:

\begin{definition}
The 1-WL refinement configuration is defined as $\mathbf{C}^{L,K}_{\text{1-WL}}=(C^{0,K}_{\text{1-WL}},C^{1,K}_{\text{1-WL}},...,C^{L,K}_{\text{1-WL}})$, where $C^{l,K}_{\text{1-WL},k}=\{l\}$ for $k=0,1$, and $C^{l,K}_{\text{1-WL},k}=\emptyset$ for others.
\end{definition}

\begin{definition}
The $K$-hop refinement configuration is defined as $\mathbf{C}^{L,K}_{\text{K-hop}}=(C^{0,K}_{\text{K-hop}},C^{1,K}_{\text{K-hop}},...,C^{L,K}_{\text{K-hop}})$, where $C^{l,K}_{\text{K-hop},k}=\{l\}$ for all $k \in [K]$.
\end{definition}

\begin{definition}
The GINE+ refinement configuration is defined as $\mathbf{C}^{L,K}_{\text{GINE+}}=(C^{0,K}_{\text{GINE+}},C^{1,K}_{\text{GINE+}},...,C^{L,K}_{\text{GINE+}})$, where $C^{l,K}_{\text{GINE+},k}=\{l\}$ for $k=0$, and $C^{l,K}_{\text{GINE+},k}=\{l-k+1\}$ for $k=1,2,...,l+1$, $C^{l,K}_{\text{GINE+},k}= \emptyset$ if $k>l+1$. 
\end{definition}
It is easy to see that $\textbf{CR}(\mathbf{C}^{L,K}_{\text{1-WL}})$ is exactly the same as the color refinement algorithm in 1-WL test. Next, we can analyze the expressive power of the general color refinement algorithm given different refinement configurations. We say $\mathbf{C}^{L,K}_{1} \succeq \mathbf{C}^{L,K}_{2}$ if $\textbf{CR}(\mathbf{C}^{L,K}_{1})$ is at least equally powerful as $\textbf{CR}(\mathbf{C}^{L,K}_{2})$ in terms of expressive power. Then, we have the following properties for the general refinement algorithm.
\begin{property}
\label{property:layer}
$\mathbf{C}^{L+1,K}\succeq \mathbf{C}^{L,K}$.
\end{property}
\begin{property}
\label{property:cover}
If $C^{l,K}_{1,k} \subseteq C^{l,K}_{2,k}$ for any $k \in [K]$ and $l\in [L]$, then $\mathbf{C}^{L,K}_{1}\succeq \mathbf{C}^{L,K}_{2}$.
\end{property}
These two properties are easy to validate, as given the injective HASH function, an algorithm with more information is always at least equally powerful as an algorithm with less information. Moreover, we have the following proposition.
\begin{proposition}
\label{pro:color_inject}
If for any $l \in [L]$ and any $k \in [K]$, $\mathbf{C}^{l,K}_{1}$ and $\mathbf{C}^{l,K}_{2}$ satisfy $i \in C^{l,K}_{1,k}$ and $j \in C^{l,K}_{2,k}$ if and only if $i \geq j$, then $\mathbf{C}^{L,K}_{1}\succeq \mathbf{C}^{L,K}_{2}$.
\end{proposition}
\begin{proof}
Let $R^{l,i}_{v,G}$ denote the color refinement result after iteration $l$ for node $v$ in graph $G=(V,E)$ using refinement configuration $\mathbf{C}^{L,K}_{i}$. 
\begin{enumerate}
\item At iteration 1, if a node aggregates its neighbor's label from some hop, it can only aggregate the initial label of nodes, or namely $R^{0,i}_{v,G}$. Then, if $C^{1,K}_{1,k}=\{\!\!\{0\}\!\!\}$, $C^{1,K}_{2,k}$ can be either $\{\!\!\{0\}\!\!\}$ or $\emptyset$. If $C^{1,K}_{1,k}=\emptyset$, then $C^{1,K}_{2,k}=\emptyset$ It is trivial to see that if $R^{1,1}_{v_1,G}=R^{1,1}_{v_2,G}$ for any pair of nodes $v_1,v_2 \in G$, then $R^{1,2}_{v_1,G}=R^{1,2}_{v_2,G}$, which means $\mathbf{C}^{0,K}_{1}\succeq \mathbf{C}^{0,K}_{2}$ holds.

\item At iteration $l$ Assume $\mathbf{C}^{l-1,K}_{1}\succeq \mathbf{C}^{l-1,K}_{2}$ holds. 

\item At iteration $l+1$, the condition in the proposition means color refinement algorithm with $C^{l,K}_{2}$ can only aggregate results from earlier iteration than $C^{l,K}_{1}$ at any hops. Meanwhile, as $\mathbf{C}^{l-1,K}_{1}\succeq \mathbf{C}^{l-1,K}_{2}$ holds, as long as $R^{l,1}_{v_1,G}=R^{l,1}_{v_2,G}$, we have $R^{t,i}_{v_1,G}=R^{t,i}_{v_2,G}$ holds for any $t \in [l]$ and $i=1,2$ given the injectiveness of HASH function. This means that if $R^{l+1,1}_{v_1,G}=R^{l+1,1}_{v_2,G}$, then $R^{l+1,2}_{v_1,G}=R^{l+1,2}_{v_2,G}$. Therefore, $\mathbf{C}^{l,K}_{1}\succeq \mathbf{C}^{l,K}_{2}$ also holds. this completes the proof.
\end{enumerate}
\end{proof}
Given two properties and Proposition~\ref{pro:color_inject}, we have the following results.
\begin{theorem}
\label{thm:color_expressive}
$\mathbf{C}^{L,K}_{\text{K-hop}} \succeq \mathbf{C}^{L,K}_{\text{GINE+}} \succeq \mathbf{C}^{L,K}_{\text{1-WL}}\succeq \mathbf{C}^{0,K}_{\text{GINE+}} \succeq \mathbf{C}^{0,K}_{\text{1-WL}}$. 
\end{theorem}

\begin{proof}
Using the Property~\ref{property:layer} and Property~\ref{property:cover} of general color refinement algorithm, it is easy to prove that $\mathbf{C}^{L,K}_{\text{K-hop}}\succeq \mathbf{C}^{L,K}_{\text{1-WL}}\succeq \mathbf{C}^{0,K}_{\text{GINE+}} \succeq \mathbf{C}^{0,K}_{\text{1-WL}}$ and $\mathbf{C}^{L,K}_{\text{GINE+}}\succeq \mathbf{C}^{L,K}_{\text{1-WL}}$. The only thing left is the comparison between $\mathbf{C}^{L,K}_{\text{K-hop}}$ and $\mathbf{C}^{L,K}_{\text{GINE+}}$. For any $l \in [L]$, $C^{l,K}_{\text{K-hop},k}=\{\!\!\{l\}\!\!\}$ for all $k \in [K]$. Instead, $C^{l,K}_{\text{GINE+},k}=\{\!\!\{l-k+1\}\!\!\}$ for $k=1,2,...,l+1$ and $C^{l,K}_{\text{GINE+},0}=\{\!\!\{l\}\!\!\}$. Then it is easy to see that $C^{l,K}_{\text{K-hop},k}$ and $C^{l,K}_{\text{GINE+},k}$ satisfy the condition of Proposition~\ref{pro:color_inject} and thus $\mathbf{C}^{L,K}_{\text{K-hop}}\succeq \mathbf{C}^{L,K}_{\text{GINE+}}$ holds.
\end{proof}
Theorem~\ref{thm:color_expressive} provide a general comparison between different color refinement configurations. Based on Theorem~\ref{thm:color_expressive}, we can actually extend the Proposition~\ref{pro:khop_power} in the main paper as
\begin{corollary}
$L$ layer $K$-hop message passing GNNs defined in Equation~(\ref{eq:khop_mp}) with the shortest path distance kernel is at least equally powerful as $L$ layer GINE+~\cite{brossard2020graph}. $L$ layer GINE+ is strictly more powerful than $L$ layer 1-hop message passing GNNs. 
\end{corollary}
The above Corollary is trivial to prove as corresponding models with injective message and update functions have at most the same expressive power as general color refinement algorithm with corresponding refinement configurations and permutation invariant readout function. 
However, one remaining question given the general color refinement algorithm is the comparison of the expressive power between $K$ layer 1-hop GNNs and 1 layer $K$-hop GNNs. Here we show that:
\begin{proposition}
\label{pro:khop1_1hopk}
Assume we use the shortest path distance kernel for $K$-hop message passing GNNs. There exists pair of graphs that can be distinguished by 1 layer $K$-hop message passing GNNs but not $K$ layer 1-hop message passing GNNs and vice versa. 
\end{proposition}

\begin{wrapfigure}{r}{0.5\textwidth}
\centering
\vspace{-5pt}
\includegraphics[width=0.5\textwidth]{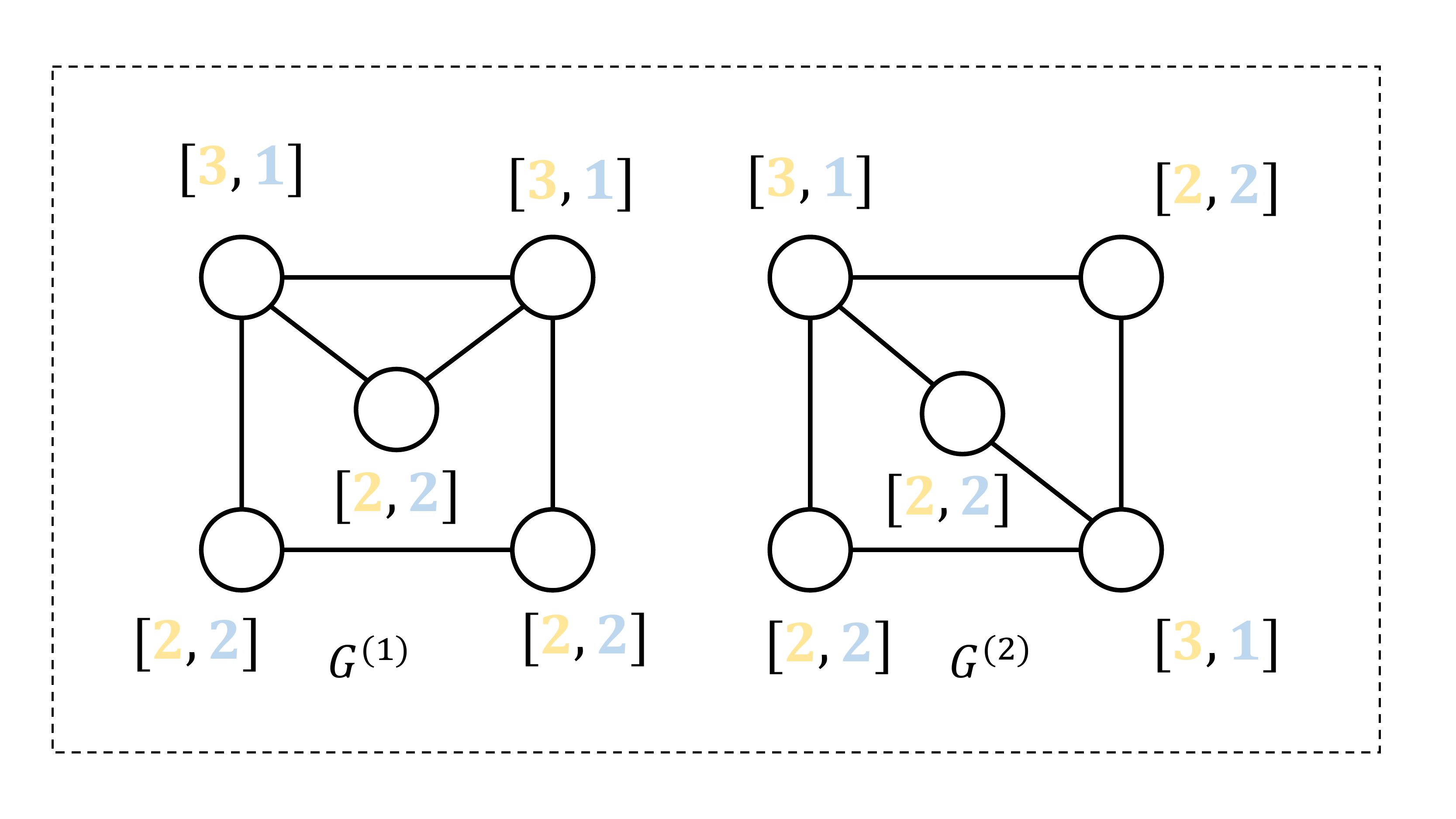}
\vspace{-20pt}
\caption{A pair of non-isomorphic graphs that can be distinguished by 2-layer 1-hop GNNs but not 1-layer 2-hop GNNs.}
 \vspace{-10pt}
\label{fig:khop1_1hopk}
\end{wrapfigure}

To prove Proposition~\ref{pro:khop1_1hopk}, we provide two examples. The first example is exactly example 2 in Figure~\ref{fig:regular1}. we know that a 1-hop GNN with 2 layers cannot distinguish two graphs as they are all regular graphs. Instead, a 1-layer 2-hop GNN is able to achieve that. From another direction, we show two graphs in Figure~\ref{fig:khop1_1hopk}. We know that 1 layer 2-hop message passing GNN is equivalent to inject node configuration up to 2-hop into nodes. As we show in Figure~\ref{fig:khop1_1hopk}, two graphs have the same node configuration set, which means an injective readout function will produce the same representation. Therefore, two graphs cannot be distinguished by 1 layer 2-hop message passing GNNs. Instead, it is easy to validate that these two graphs can be distinguished by 2 layer 1-hop message passing GNNs. 

\subsection{Discussion on existing $K$-hop models}
\textbf{GINE+}~\citep{brossard2020graph}: GINE+ tries to increase the representation power of graph convolution by increasing the kernel size of convolution. However, at $l$-th layer of GINE+, it only aggregates information from the neighbors of hop $1,2,...,l$ after $l-1, l-2, ..., 0$ layer, which means that after $L$ layer of convolution, the GINE+ still has a receptive field of size $L$. As we discussed in the previous section, $L$ layer $K$-hop message passing GNNs with the shortest path distance kernel is at least equally powerful as $L$ layer GINE+.

\textbf{Graphormer}~\citep{ying2021do}: Graphormer introduce a new way to apply transformer architecture~\citep{vaswani2017attention} on graph data. In each layer of Graphormer, the shortest path distance is used as spatial encoding to adjust the attention score between each node pair. Although the Graphormer does not apply the $K$-hop message passing directly, the attention mechanism (each node can see all the nodes) with the shortest path distance feature implicitly encodes a rooted subtree similar to the $K$-hop message passing with the shortest path distance kernel. To see it clearly, suppose we have $K$-hop message passing with $K=\infty$ and graphs only have one connected component. It will aggregate all the nodes at each layer, which is similar to Graphormer. Meanwhile, the injective message and update function implicitly encode the shortest path distance to each node in aggregation. Then it is trivial to see that Graphormer is actually a special case of $K$-hop message passing with the shortest path distance kernel. 

\textbf{Spectral GNNs}: spectral-based GNNs serve as an important type of graph neural network and gain lots of interest in recent years. Here we only consider one layer as spectral-based GNNs usually only use 1 layer. the general spectral GNNs can be written as:
\begin{align}
\label{eq:spectral}
\begin{aligned}
Z=\phi \left(\sum_{k=0}^K\alpha_k\rho(\hat{L}^k) \varphi (x)\right),~
\end{aligned}
\end{align}

where $\phi$ and $\varphi$ are typically multi-layer perceptrons (MLPs), $\hat{L}$ is normalized Laplacian matrix, $\alpha_k$ is weight for each spectral basis, and $\rho$ is an element-wise function of matrix. In normal spectral GNNs, $\rho$ is always an identity mapping function. We can see spectral GNNs have a close relationship to $K$-hop message passing with graph diffusion kernel as $\hat{L}^k$ actually compute the $Q^{k,gd}_{v,G}$ for each node by only keeping element in the matrix that is larger than 0. As $K$ in Equation~(\ref{eq:spectral}) can be greater than 1, it looks like spectral GNNs fit Proposition 1 as well. However, according to Proposition 4.3 in~\cite{wang2022powerful}, all such spectral-based methods have expressive power no more than the 1-WL test. This seems like a discrepancy. The reason lies in the $\rho$ function in Equation~(\ref{eq:spectral}). $K$-hop message passing with graph diffusion kernel can be regarded as using a non-linear $\rho$ function shown in the following:
\begin{align}
\begin{aligned}
\label{eq:non_linear}
\rho(x)=\left\{ 
    \begin{array}{lc}
        1 &x > 0 \\
        0 &others\\
    \end{array}
\right.
\end{aligned}
\end{align}
This non-linear function is the key difference between the normal spectral GNNs and $K$-hop message passing and it endows normal spectral GNNs with extra expressive power than the 1-WL test.

\textbf{\textbf{MixHop}~\citep{abu2019mixhop}, \textbf{GPR-GNN}~\citep{chien2021adaptive}, and \textbf{MAGNA}~\citep{wang2021multihop}}: All three methods extend the scope of 1-hop message passing by considering multiple graph diffusion step simultaneously. However, all these methods use normal spectral GNNs as the base encoder and thus they obey Proposition 4.3 in~\cite{wang2022powerful} instead of Proposition 1. It can be seen as a "weak" version of $K$-hop message passing with graph diffusion kernel.
\section{Comparison between $K$-hop message passing and Distance Encoding}
\label{app:khop_de}
In this section, we further discuss the connection and difference between $K$-hop message passing and Distance Encoding~\citep{li2020distance}. Here we assume that the kernel of $K$-hop message passing is the shortest path distance kernel and Distance Encoding with the shortest path distance as distance feature. 

\begin{figure*}[t]
\centering
% \vspace{5pt}
\includegraphics[width=0.95\textwidth]{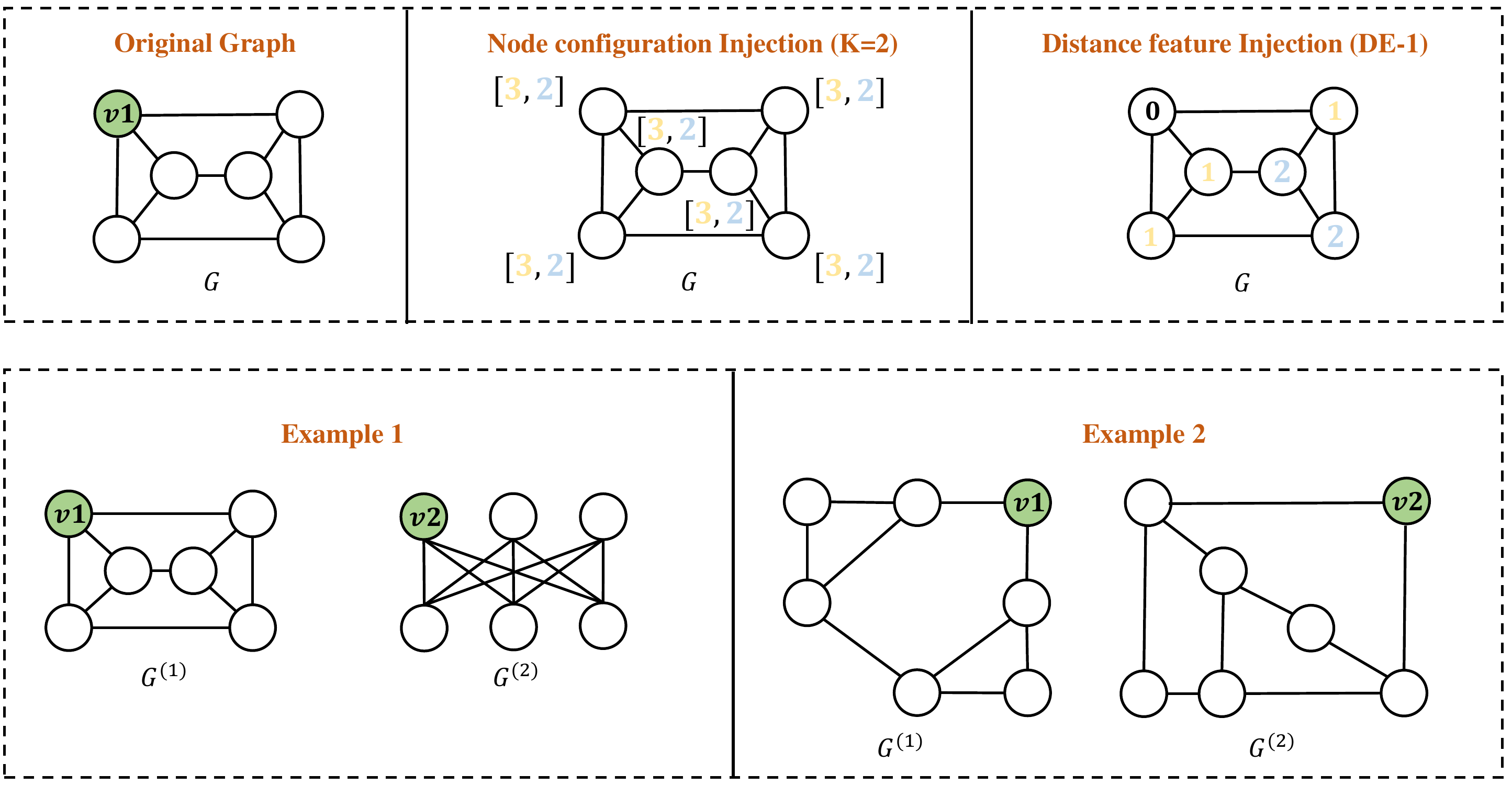}
\vspace{-5pt}
\caption{The upper part: graph with node configuration as the injected label and DE-1 as the injected label on the center node. The bottom part: two pairs of non-isomorphic graphs where the node pair in example 1 can be distinguished by DE-1 and the node pair in example 2 can be distinguished by $K$-hop message passing.}
 \vspace{-15pt}
\label{fig:khop_de}
\end{figure*}

To simplify the discussion, suppose there are two graphs $G^{(1)}=(V^{(1)},E^{(1)})$ and $G^{(2)}=(V^{(2)},E^{(2)})$. We pick two nodes $v_1$ and $v_2$ from each graph respectively and learn the representation for these two nodes. First, let us consider what 1 layer $K$-hop message passing and DE-1 without message passing. As we stated in the main paper, 1 layer of $K$-hop message passing actually injects each node with the label of node configuration. Instead, DE-1 injects each node with the distance to the center node ($v_1$ and $v_2$ here). Then we can see there is a clear difference between the two methods even if they all implicitly or explicitly use the distance feature as shown in the upper part of Figure~\ref{fig:khop_de}. Next, it is easy to see that applying $L+1$ layer $K$-hop message passing is equivalent to applying $L$ layer $K$-hop message passing on a graph with node configuration as the initial label. Applying $L$ layer DE-1 is equivalent to applying $L$ layer 1-hop message passing on a graph with the distance to the center node as the initial label. In the following, we show that the two methods cannot cover each other in terms of distinguishing different nodes:
\begin{proposition}
\label{pro:khop_de}
For two non-isomorphic graphs $G^{(1)}=(V^{(1)},E^{(1)})$ and $G^{(2)}=(V^{(2)},E^{(2)})$, we pick two nodes $v1$ and $v_2$ from each graph respectively. Then there exist pairs of graphs that nodes $v_1$ and $v_2$ can be distinguished by $K$-hop message passing with the shortest path distance kernel but not DE-1, and vice versa.
\end{proposition}
To prove the Proposition~\ref{pro:khop_de}, we provide two examples in the lower part of Figure~\ref{fig:khop_de}. Example 1 in the figure is exactly the same pair of regular graphs in example 1 of Figure~\ref{fig:regular1}. As discussed before, $K$-hop message passing with the shortest path distance kernel cannot distinguish node $v_1$ and $v_2$. However, after injecting the distance feature to the center node in two graphs and performing 2 layers of message passing, DE-1 is able to assign different representations for two nodes. In example 2, DE-1 cannot distinguish nodes $v_1$ and $v_2$ in the two graphs, even if they are not regular graphs. Instead, using $K$-hop message passing with only $K=2$, the two nodes will get different representations. We omit the detailed procedure here as it is easy to validate. Using these two examples, we have shown that $K$-hop message passing and DE-1 are not equivalent to each other even if they all use the shortest path distance feature. The root reason is that although DE-1 uses the distance feature, it still aggregates information only from 1-hop neighbors in each iteration, while $K$-hop message passing directly aggregates information from all $K$-hop neighbors. That is, their ways of using distance information are different. However, we also notice that in example 2, if we label the graph with DE-1 on another pair of nodes, two nodes can be distinguished, which means DE-1 is able to distinguish these two graphs. \textbf{However, to achieve that, DE-1 need to label the graph $n$ times and run the message passing on all $n$ labeled graphs. Instead, $K$-hop message passing only needs to run the message passing once, which is both space and time efficient.}

Besides DE-1, ~\citet{li2020distance} also proposed the DEA-GNN which is at least no less powerful than DE-1. The DEA-GNN extends DE-1 by simultaneously aggregating all other nodes in the graph but the message is encoded with the distance to the center node. This can be seen as exactly performing $K$-hop message passing. In other words, DEA-GNN is the combination of $K$-hop message passing and DE-1. Therefore it is easy to see:
\begin{proposition}
The DEA-GNN~\citep{li2020distance} is at least equally powerful as $K$-hop message passing with the shortest path distance kernel.
\end{proposition}

\section{Proof of Theorem~\ref{thm:khop_limitation}}\label{app:khop_limitation}
In this section, we prove Theorem~\ref{thm:khop_limitation}. We first restate Theorem~\ref{thm:khop_limitation}: The expressive power of a proper $K$-hop message passing GNN of any kernel is bounded by 3-WL test. Our proof is inspired by the recent results in SUN~\citep{Frasca2022UnderstandingAE}, which bound all subgraph-based GNN with the 3-WL test by proving that all such methods can be implemented by 3-IGN. Here we prove that $K$-hop message passing can also be implemented by 3-IGN. We will not discuss the detail of 3-IGN and all its operations. Instead, we directly follow all the definitions and notations and refer readers to Appendix B of~\cite{Frasca2022UnderstandingAE} for more details.

\textbf{$K$-hop neighbor extraction}: To implement the $K$-hop message passing, we first implement the extraction of $K$-hop neighbors. The key insight is that we can use the $l$-th  channel in $X_{iij}$ to store whether node $j$ is the neighbor of node $i$ at $l$-th hop, which is similar to the extraction of ego-network. Here we suppose $d\ge K$. Same as all node selection policies, we first lift the adjacency $A$ to a three-way tensor $\mathcal{Y}\in \mathbb{R}^{n^3\times d}$ using broadcasting operations:

\begin{align}
X_{i i i}^{(0)}=\boldsymbol{\beta}_{i, i, i} A_{i i}. \\
X_{j i i}^{(0)}=\boldsymbol{\beta}_{*, i, i} A_{i i}. \\
X_{i i j}^{(0)}=\boldsymbol{\beta}_{i, i, j} A_{i j}. \\
X_{i j i}^{(0)}=\boldsymbol{\beta}_{i, j, i} A_{i j}. \\
X_{k i j}^{(0)}=\boldsymbol{\beta}_{*, i, j} A_{i j}.
\end{align}

Now, $X_{i i j}^{(0),1}$ store the 1-hop neighbors of node $i$. Next, $l$-th hop neighbor of node $i$ is computed and stored in $X_{i i j}^{,l}$. To get neighbor of $l$-th hop for $l=2,3,...,K$, we first copy it into $d+1$ channels:

\begin{align}
    X_{i j j}^{(1)}=\boldsymbol{\kappa}_{: d /(d+1)}^{: d} X_{i j j}^{(0)}+\boldsymbol{\kappa}_{d+1:d+1}^{l: l} \boldsymbol{\beta}_{i, j, j} X_{i i j}^{(0)}.
\end{align}

The $d+1$-th channel is used to compute high-order neighbors. Next, we extract all $K$ hop neighbors by iteratively following steps $K-1$ times and describe the generic $l$-th step. We first broadcast the current reachability pattern into $X_{ijk}$, writing it into the $l$-th channel:

\begin{align}
\label{eq:12}
    X_{i j k}^{(l, 1)}=X_{i j k}^{(l-1)}+\boldsymbol{\kappa}_{l: l}^{d+1:d+1} \boldsymbol{\beta}_{i, *, j} X_{i j j}^{(l-1)}.
\end{align}

Then, a logical AND is performed to get the new reachability. Then write back the results into the $l$-th channel:

\begin{align}
\label{eq:13}
    X_{i j k}^{(l, 2)}=\boldsymbol{\varphi}_{l:l /(d+1)}^{(\wedge) 1, l} X_{i j k}^{(l, 1)}.
\end{align}

Next, we get all nodes that can be reached within $l$ hops by performing pooling, clipping, and copying back to $d+1$ channel:

\begin{align}
X_{i j j}^{(l, 3)} &=\boldsymbol{\kappa}_{:d/(d+1)}^{:d}X_{ijj}^{(l,2)}+\boldsymbol{\kappa}^{l:l}_{d+1:}\boldsymbol{\beta}_{i,j,j}\boldsymbol{\pi}_{i,j}X_{ijk}^{(l,2)}. \\
X_{i j j}^{(l,4)} &=\left[\begin{array}{ll} \boldsymbol{\kappa}_{:d/(d+1)}^{:d} & \boldsymbol{\varphi}_{d+1: d+1}^{(\downarrow) d+1: d+1} \end{array}\right]  X_{i j j}^{(l, 3)}.
\end{align}

Now $X_{ijj}^{(l,4),d+1}$ save all the nodes that can be reached at $l$-th hop. Finally, we extract the $l$-th hop neighbor and copy it into $l$-th channel. For graph diffusion kernel, the current result is itself result for graph diffusion kernel, which means we only need to copy it:

\begin{align}
X_{iij}^{(l)}=X_{iij}^{(l,4)}+\boldsymbol{\kappa}_{l:l}^{d+1:d+1}\boldsymbol{\beta}_{i,i,j}X_{ijj}^{(l,4)}.
\end{align}

For the shortest path distance kernel, we need to nullify all the nodes that already existed in the previous hops. To achieve this, we need to first compute if a node exists in the previous hops and then nullify it:

\begin{align}
X_{iij}^{(l,5)}&=X_{iij}^{(l,4)}+\boldsymbol{\kappa}_{l:l}^{d+1:d+1}\boldsymbol{\beta}_{i,i,j}X_{ijj}^{(l,4)}.\\
X_{iij}^{(l,6)}&=\sum_{i=1}^{l-1}\boldsymbol{\kappa}_{d+1:d+1}^{i:i}X_{iij}^{(l,5)}.\\
X_{iij}^{(l)}&=X_{iij}^{(l,4)}+\boldsymbol{\varphi}_{l: l}^{(\wedge) l,d+1}\boldsymbol{\varphi}_{d+1}^{(!)d+1}X_{iij}^{(l,6)}.
\end{align}

Where $\boldsymbol{\varphi}_{b}^{(!)a}$ is logical not function that output 0 if input is not 0 and vice versa for channel $a$ and write result into channel $b$. Here we omit the detailed implementation as it is easy to implement using ReLu function. Finally, we bring all other orbits tensors to the original dimensions:

\begin{align}
X_{iii}&=\boldsymbol{\kappa}_{:d}^{:d}X_{iii}^{(l)}.\\
X_{ijj}&=\boldsymbol{\kappa}_{:d}^{:d}X_{ijj}^{(l)}.\\
X_{iij}&=\boldsymbol{\kappa}_{:d}^{:d}X_{iij}^{(l)}.\\
X_{iji}&=\boldsymbol{\kappa}_{:d}^{:d}X_{iji}^{(l)}.\\
X_{ijk}&=\boldsymbol{\kappa}^{:1}_{:1/(d)}X_{ijk}^{(l)}.
\end{align}

Now, we have successfully implemented $K$-hop neighbor extraction algorithm using 3-IGN.

\textbf{$K$-hop message passing}:
To implement the message and update function for each layer, we use the same base encoder as~\cite{morris2019weisfeiler}. Other types of base encoders and combine functions can be implemented in a similar way. We follow the same procedure of implementing the base encoder of DSS-GNN as stated in~\cite{Frasca2022UnderstandingAE}. Note here for each hop the procedure is similar therefore we state the generic $l$-th hop. The key insight is that, in $K$-hop message passing, we are actually not working on each subgraph but the original graph, which means all the operations can be implemented in the orbit tensor $X_{iii}$ and $X_{iij}$.

\textit{Message broadcasting}: This procedure is similar to the base encoder of DSS-GNN but here we only need to broadcast $X_{iij}$. However, since we need to perform message passing for $K$ times, we need to broadcast it $K$ times:

\begin{align}
X_{iij}^{(t,1)}=\boldsymbol{\kappa}_{:d/((K+1)d)}^{:d}X_{iij}^{(t)}+\sum_{i=1}^{K}\boldsymbol{\kappa}^{:}_{id+1:(i+1)d/(K+1)d}\boldsymbol{\beta}_{iij}X_{ijj}^{(t)}.
\end{align}

\textit{Message sparsification and aggregation}: Similar to DSS-GNN, here we need to nullify the message from nodes that are not $l$-th hop neighbors. Here we define the following function:

\begin{align}
    f_{i i j}^{\odot}\left(X_{aab}^{(t, 1)}\right)_l= \begin{cases}\mathbf{0}_{d} & \text { if } X_{a a b}^{(t, 1), l}=0, \\ X_{aab}^{(t, 1),l(d+1):(l+1)d } & \text { otherwise. }\end{cases}
\end{align}

The function is used to nullify the message for $l$-th hop. It is easy to validate the existence of such function following the same procedure in~\cite{Frasca2022UnderstandingAE} therefore we omit the detail. Next, the message sparsification can be implemented by:

\begin{align}
X_{iij}^{(t,3)}=\left[\begin{array}{llll}\boldsymbol{\kappa}^{:d}_{:d} & \boldsymbol{\varphi}_{1}^{\left(\odot_{i j j}\right):}&...& \boldsymbol{\varphi}_{K}^{\left(\odot_{i j j}\right):}\end{array}\right] X_{i i j}^{(t, 2)}.
\end{align}

Then, the message function for $K$-hop message passing is:

\begin{align}
X_{iii}^{(t,4)}=\boldsymbol{\kappa}^{:d}_{:d}X_{iii}^{(t,3)}+\boldsymbol{\kappa}^{d+1:}_{d+1:}\boldsymbol{\beta}_{iii}\boldsymbol{\pi}_{i}X^{(t,3)}_{ijj}.
\end{align}

\textit{Update} Then, the update function is implemented using linear transformation. In $K$-hop message passing, each hop needs an independent parameter set. Here in order to operate on constructed tensor, we define $W_{t}^{l}=\left[W_{t,1}^{l}|| W_{t,2}^{l}||...||W_{t,K}^{l}\right]$, where $W_{t,i}^{l}=\boldsymbol{0_d}$ if $i\neq l$. Then, the update function is:

\begin{align}
    X_{iii}^{(t,5)}=\sum_{l=1}^{K}\sigma\left(W_{t}^{l}X_{iii}^{(t,4)}\right).
\end{align}

\textit{Combine}: Finally, we implement the sum combine function. The combine function is can be implemented by a simple MLP:

\begin{align}
        X_{iii}^{(t+1)}=\boldsymbol{\varphi}_{d}^{\boldsymbol{f}}(X_{iii}^{(t,4)}).\\
        X_{jii}^{(t+1)}=\boldsymbol{\beta}_{jii}X_{iii}^{t+1}.
\end{align}

Now, we successfully implement both $K$-hop neighbor extraction and $K$-hop message passing layer. Note that other parts like the readout function can be easily implemented and we omit the detail. This means the expressive power of $K$-hop message passing with either graph diffusion kernel or shortest path distance kernel is bounded by 3-IGN. Based on the Theorem 2 in~\cite{Geerts2020TheEP}, we conclude the Theorem~\ref{thm:khop_limitation}. Further, it is intuitive that all methods that use the node pair distance feature are also bounded by 3-WL as this feature can be computed by 3-IGN. We leave the formal proof in our future work.
\section{Proof and discussion of Proposition~\ref{pro:kp_drg}}
\label{app:kp_expressive}
Proposition~\ref{pro:kp_drg} seems intuitive and easy to prove. However, it also unclear how powerful are KP-GNN. So here we still give a more detailed discussion on it. we first give the definition of \textit{distance regular graph}.
\begin{definition}
\textit{A distance regular graph is a regular graph such that for any two nodes $v$ and $u$, the number of nodes with the distance of $i$ from $v$ and distance of $j$ from $u$ depends only on $i$, $j$ and the distance between $v$ and $u$.}
\end{definition}
Furthermore, we only consider the connected distance regular graphs with no multi-edge or self-loop. Such graphs can be characterized by \textit{intersection array}. 
\begin{definition}
\textit{The intersection array of a distance regular graph with diameter $d$ is an array of integers $(b_0,b_1,...,b_{d-1};c_1,c_2,...,c_d)$, where for all $1\leq j \leq d$, $b_j$ gives the number of neighbors of $u$ with distance $j+1$ from $v$ and $c_j$ gives the number of neighbors of $u$ with distance $j-1$ from $v$ for any pair of $(v,u)$ in graph with distance $j$.}
\end{definition}
Given the definition of distance regular graph and intersection array, we can propose the first lemma. 
\begin{lemma}
\label{lem:drg_sg}
\textit{Given a distance regular graph $G$ with intersection array $(b_0,b_1,...,b_{d-1};c_1,c_2,...,c_d)$. Pick a node $v$ from $G$, the peripheral subgraph $G^{j,spd}_{v,G}$ is a $n$-sized $r$-regular graph with $n=|Q^{j,spd}_{v,G}|$ and $r=b_0-b_j-c_j$}
\end{lemma}
\begin{proof}
Given a distance regular graph $G$ with intersection array $(b_0,b_1,...,b_{d-1};c_1,c_2,...,c_d)$, from the definition of intersection array, for node $v$ in $G$, $Q^{j,spd}_{v,G}$ is the node set that have distance $j$ from node $v$. Then, $b_j$ is the number of neighbors for each node in $Q^{j,spd}_{v,G}$ that have distance $j+1$ to node $v$. It is easy to see that these neighbors must belong to $Q^{j+1,spd}_{v,G}$, which means that $b_j$ is also the number of edge for a node in $Q^{j,spd}_{v,G}$ that connect to nodes in $Q^{j+1,spd}_{v,G}$. Similarly, $c_j$ is the number of edge for a node in $Q^{j,spd}_{v,G}$ that connect to nodes in $Q^{j-1,spd}_{v,G}$. For node $u \in Q^{j,spd}_{v,G}$, we know that the edges of $u$ must connect to either $Q^{j+1,spd}_{v,G}$, $Q^{j,spd}_{v,G}$, or $Q^{j-1,spd}_{v,G}$. Since the degree of node $u$ is $b_0$, then the number of edge that connect node $u$ to nodes in $Q^{j,spd}_{v,G}$ is $b_0-b_j-c_j$. The above statement holds for each $u\in Q^{j,spd}_{v,G}$, which means all nodes $u\in Q^{j,spd}_{v,G}$ have same node degree. Meanwhile, the node set of peripheral subgraph $G^{j,spd}_{v,G}$ is exactly $Q^{j,spd}_{v,G}$. Combine two statements, we can conclude that $G^{j,spd}_{v,G}$ is a $n$-sized $r$-regular graph with $n=|Q^{j,spd}_{v,G}|$ and $r=b_0-b_j-c_j$.
\end{proof}
Given the Lemma~\ref{lem:drg_sg}, we know that the peripheral subgraph of a node in any distance regular graph is itself a regular graph. 
Next, given two non-isomorphic distance regular graphs $G^{(1)}$ and $G^{(2)}$ with the same intersection array, there are total $d$ pairs of regular peripheral subgraphs for any pair of nodes $v_1$ and $v_2$. If the KP-GNN can distinguish two regular graphs at some hop $j\leq d$, then the KP-GNN can distinguish $v_1$ and $v_2$ in graph $G^{(1)}$ and $G^{(2)}$. Meanwhile, it is easy to see that each node in a distance regular graph has the same local structure. Therefore, as long as $v_1$ can be distinguished from $v_2$, KP-GNN can distinguish two graphs. 

Now given the implementation of KP-GNN in Equation~\ref{eq:KP-GNN_imp}, if either the peripheral set $E(Q^{k,t}_{v_1,G^{(1)}})$ is different from $E(Q^{k,t}_{v_2,G^{(2)}})$ or peripheral configuration  $C^{k^{\prime}}_{j,G^{(1)}}$ is different from $C^{k^{\prime}}_{j,G^{(2)}}$. However, as peripheral subgraph is itself regular graphs, $E(Q^{k,t}_{v_1,G^{(1)}})$ must equal to $E(Q^{k,t}_{v_2,G^{(2)}})$. Therefore, KP-GNN can only distinguish two peripheral subgraph by $C^{k^{\prime}}_{j,G^{(1)}}$ and $C^{k^{\prime}}_{j,G^{(2)}}$.

It looks like we can use the result from Theorem~\ref{thm:khop_regular} to prove that Equation~\ref{eq:KP-GNN_imp} can distinguish almost all distance regular graphs. However, although the peripheral subgraphs of $v_1$ and $v_2$ are regular graphs with the same size and $r$, they are not randomly generated by the configuration model and do not satisfy Theorem~\ref{thm:khop_regular}. We leave the quantitative expressive power analysis of KP-GNN in our future works. 

\section{Time, space complexity and limitation of $K$-hop message passing and KP-GNN}
\label{app:complexity}
\subsection{Time and space complexity}
In this section, we analyze the time and space complexity. To simplify the analysis, we first consider the shortest path distance kernel, as the graph diffusion kernel has both time and space complexity no less than the shortest path distance kernel. Denote graph $G$ with $n$ node and $m$ edges. 

\textbf{Space complexity}: For both $K$-hop message passing and KP-GNN, as we only need to maintain one representation for each node, the space complexity is $O(n)$ like in vanilla $1$-hop message passing.

\textbf{Time complexity}: First, we analyze the time complexity of $K$-hop message passing GNNs. For each node, suppose in the worst case we extract neighbors from all nodes from all hops, the number of neighbors we need to aggregate from all hops is $n$ (all nodes in the graph). Then, $n$ nodes need at most $O(n^2)$ time complexity. Next, we analyze the time complexity of KP-GNN. The additional time complexity comes from counting the peripheral edges and $k^{\prime}$-configuration. Since the counting can be done in a preprocessing step and reused at each message passing iteration, it will be amortized to zero finally. So the practical time complexity is still $O(n^2)$.

\subsection{Discussion on the complexity}

From the above analysis, we know that $K$-hop GNNs and KP-GNN only need $O(n)$ space complexity, which is equal to vanilla message passing GNNs and much less than subgraph-based GNNs which require $O(n^2)$ space (due to maintaining a different representation for each node when appearing in different subgraphs). Thus, KP-GNN has a much better space complexity than subgraph-based GNNs, PPGN~\citep{maron2019provably} (also $O(n^2)$) and 3-WL-GNNs~\cite{maron2018invariant} ($O(n^3)$).

The worst-case time complexity of $K$-hop GNNs and KP-GNN is much higher than that of normal GNNs due to aggregating information from more than 1-hop nodes in each iteration. However, we also note that the larger receptive field could reduce the number of message passing iterations because 1-layer of $K$-hop message passing can cover the receptive field of $K$-layer 1-hop message passing. Furthermore, $K$-hop GNN and KP-GNN have better time complexity $O(n^2)$ than that of subgraph-based GNNs (which require $O(nm)$ for doing 1-hop message passing $O(m)$ in all $n$ subgraphs). For sparse graphs, we can already save a factor of $d_{\text{avg}}=m/n$ complexity. For dense graphs, the worst-case time complexity of subgraph-based GNNs becomes $O(n \cdot n^2) = O(n^3)$, and our time complexity advantage becomes even more significant.

\subsection{Limitations}
We discuss the limitation of the proposed KP-GNN from two aspects. 

\textbf{Stability}: Using K-hop instead of 1-hop can make the receptive field of a node increase with respect to $K$. For example, to compute the representation of a node with $L$ layer $K$-hop, GNN, the node will get information from all $LK$-hop neighbors. The increased receptive field can hurt learning, as mentioned in GINE+~\citep{brossard2020graph}. It is an intrinsic limitation that exists in all $K$-hop GNNs. GINE+~\citep{brossard2020graph} proposed a new way to fix the receptive field as $L$ by only considering $L-i$ layer representation of neighbor in $i$ hop during the aggregation. We also apply this approach and it helps mitigate the issue and shows great practical performance gain. Further, we propose a variant of KP-GNN named KP-GNN$'$, which only run KP-GNN at the first layer but 1-hop message passing at the rest of the layers. KP-GNN$'$ help mitigates stability issue and can be applied to large $K$ without causing the training of the model to fail. It achieves great results in various real-world datasets.

\textbf{Time complexity}: As we show above, we need $O(n^2)$ time complexity for KP-GNN, which is much higher than $O(m)$ of MPNN. However, this limitation exists for all subgraph-based methods like NGNN~\citep{zhang2021nested}, GNN-AK~\citep{zhao2022from}, and ESAN~\citep{bevilacqua2022equivariant} as they all require $O(nm)$ time complexity~\citep{Frasca2022UnderstandingAE}. This is the sacrifice for better expressive power. 

\section{Implementation detail of KP-GNN}\label{app:implementation}
In this section, we discuss the implementation detail of KP-GNN.

\textbf{Combine function}: 1-hop message passing GNNs do not have $\text{COMBINE}^l$ function. Here we introduce two different $\text{COMBINE}^l$ functions. The first one is the attention~\citep{luong-etal-2015-effective} based combination mechanism, which automatically learns the importance of representation for each node at each hop. The second one uses the well-known geometric distribution~\citep{wang2021multihop}. The weight of hop $i$ is computed based on $\theta_i=\alpha(1-\alpha)^i$, where $\alpha\in (0,1]$. The final representation is calculated by weighted summation of the representation of all the hops. In our implementation, the $\alpha$ is learnable and different for each feature channel and each hop.

\textbf{Peripheral subgraph information}: In the current implementation, KP-GNN will compute two pieces of information. The first one is the peripheral edge set $E(Q^{k,t}_{v,G})$. In our implementation, we compute the number of edges for each distinct edge type. The second one is $k^{\prime}$ configuration $C^{k^{\prime}}_k$, which contains node configuration and peripheral edges for all nodes in the peripheral subgraph. this is equivalent to running 1 layer of KP-GNN on each peripheral subgraph. Specifically, for node configuration, we compute the node configuration for each node and sum them up. For the peripheral edge set, we compute the total number of edges across all hops (do not consider edge type here). Finally, these two pieces of information are combined as $C^{k^{\prime}}_k$. All these steps are down in the preprocessing stage.

\textbf{KP-GCN, KP-GIN, and KP-GraphSAGE}: We implement KP-GCN, KP-GIN, and KP-GraphSAGE using the message and update function defined in GCN~\citep{kipf2017semisupervised}, GIN~\citep{xu2018powerful}, and GraphSAGE~\citep{hamilton2017inductive} respectively. In each hop, independent parameter sets are used and the computation strictly follows the corresponding model. However, increasing the number of $K$ will also increase the total number of parameters, which is not scalable to $K$. To avoid this issue, we design the $K$-hop message passing in the following way. Suppose the total hidden size of the model is $H$, the hidden size of each hop is $H/K$. In this way, the model size is still on the same scale even with large $K$. 

\textbf{KP-GIN+}: In a normal $K$-hop message passing framework, all $K$-hop neighbors will be aggregated for each node. It means that, after $L$ layers, the receptive field of GNN is $LK$. This may cause the unstable of the training as unrelated information may be aggregated. To alleviate this issue, we adapted the idea from GINE+~\cite{brossard2020graph}. Specifically, we implement KP-GIN+, which applies exactly the same architecture as GINE+ except here we add peripheral subgraph information. At layer $l$, GINE+ only aggregates information from neighbors within $l$-hop, which makes a $L$ layer GINE+ still have a receptive field of $L$. Note that in KP-GIN+, we use a shared parameter set for each hop. 

\textbf{KP-GIN$'$}: In KP-GIN$'$, we run a simple version of KP-GIN, which only uses KP-GIN at the first layer, but normal GIN for the rest of the layer. Although KP-GIN$'$ is weaker than KP-GIN from the expressiveness perspective. However, it is much more stable than KP-GIN in real-world datasets. Further, as we only have KP-GIN at the first layer, we can go larger $K$ without the cost of time complexity during training and inference. We observe great empirical results of KP-GIN$'$ on real-world datasets with less time complexity than normal KP-GNN.  

\textbf{Path encoding}: To further utilize the graph structure information on each hop, we introduce the path encoding to KP-GNN. Specifically, we not only count whether two nodes are neighbors at hop $k$, but also count the number of walks with length $k$ between two nodes. Such information can be obtained with no additional cost as the $A^k$ of a graph $G$ with adjacency $A$ is a walk-counter with length $k$. Then the information is added to the $AGG^{l,normal}_{k}$ function as additional features.

\textbf{Other implementation}: For all GNNs, we apply the Jumping Knowledge method~\citep{xu2018representation} to get the final node representation. The possible methods include sum, mean, concatenation, last, and attention. Batch normalization is used after each layer. For the pooling function, we implement mean, max, sum, and attention and different tasks may use different readout function. 

\section{Experiential setting details}\label{app:experiential}
\textbf{EXP dataset}: For both K-GIN and KP-GIN, we use a hidden size of $48$. The final node representation is output from the last layer and the pooling method is the summation. In the experiment, we use 10-fold cross-validation. For each fold, we use 8 folds for training, 1 fold for validation, and 1 fold for testing. We select the model with the best validation accuracy and report the mean results across all folds. The training epoch is set to 40. In this experiment, we do not use path encoding for a fair comparison. The learning rate is set to 0.001 and we use \textit{ReduceLROnPlateau} learning rate scheduler with patience of 5 and a reduction factor of 0.5. 

\textbf{SR25 dataset}: For both K-GIN and KP-GIN we use a hidden size of $64$. The final node representation is output from the last layer and the pooling method is the summation. For SR25 dataset, we directly train the validate the model on the whole dataset and report the best performance across 200 epochs. For each fold, we use 8 folds for training, 1 fold for validation, and 1 fold for testing. We select the model with the best validation accuracy and report the mean results across all folds. The training epoch is set to 200. In this experiment, we do not use path encoding for a fair comparison. The learning rate is set to 0.001. 

\textbf{CSL dataset}: For both K-GIN and KP-GIN we use a hidden size of $48$. The final node representation is output from the last layer and the pooling method is the summation. In the experiment, we use 10-fold cross-validation.

\textbf{Graph\&Node property dataset}: 
For graph and node property prediction tasks, we train models with independent 4 runs and report the mean results. For both K-GIN+ and KP-GIN+, the hidden size of models is set as 96. The final node representation is the concatenation of each layer. The pooling method is attention for graph property prediction tasks and sum for node property prediction tasks. The learning rate is 0.01 and we use \textit{ReduceLROnPlateau} learning rate scheduler with patience of 10 and a reduction factor of 0.5. We use the shortest path distance kernel. The maximum number of epochs for each run is 250. For KP-GIN+, we search $K$ from 3 to 6 with/without path encoding and report the best result. For K-GIN+, we search $K$ from 3 to 6 without path encoding and report the best result.  

\textbf{Graph substructure counting dataset}: 
For graph substructure counting tasks, we train models with independent 4 runs and report the mean results. For both K-GIN+ and KP-GIN+, the hidden size of models is set as 96. The final node representation is the concatenation of each layer. The pooling method is the summation. The learning rate is 0.01 and we use \textit{ReduceLROnPlateau} learning rate scheduler with patience of 10 and a reduction factor of 0.5. We use the shortest path distance kernel. The maximum number of epochs for each run is 250. For KP-GIN+, we search $K$ from 1 to 4 with/without path encoding and report the best result. For K-GIN+, we search $K$ from 1 to 4 without path encoding and report the best result.

\textbf{TU datasets}: For TU datasets, we use 10-fold cross-validation. We report results for both settings in~\cite{xu2018powerful} and~\cite{wijesinghe2022a}. For the first setting, we use 9 folds for training and 1 fold for testing in each fold. After training, we average the test accuracy across all the folds. Then a single epoch with the best mean accuracy and the standard deviation is reported. For the second setting, we still use 9 folds for training and 1 fold for testing in each fold but we directly report the mean best test results. For KP-GNN, we implement GCN~\citep{kipf2017semisupervised}, GraphSAGE~\citep{hamilton2017inductive} and GIN~\citep{xu2018powerful} version.  we search (1) the number of layer $\{2,3,4\}$, (2) the number of hop $\{2,3,4\}$, (3) the kernel of K-hop $\{spd,gd\}$ , and (4) the $COMBINE$ function $\{attention, geometric\}$. The hidden size is 33 when $K=3$ and 32 for the rest of the experiments. The final node representation is the last layer and the pooling method is the summation. The dropout rate is set as 0.5, the number of training epochs for each fold is 350 and the batch size is 32. The initial learning rate is set as $1e-3$ and decays with a factor of 0.5 after every 50 epochs. 

\textbf{QM9 dataset}: For QM9 dataset, we implement KP-GIN+ and KP-GIN$'$. For both two models, the hidden size of the model is 128. The final node representation is the concatenation of each layer and the pooling method is attention. The dropout rate is 0. For KP-GIN+, we use the shortest path distance kernel with $K=8$ and 8 layers. Meanwhile, we add the virtual node to the model. For KP-GIN$'$, we use the shortest path distance kernel with $K=16$ and 16 layers. Meanwhile, we add the residual connection. The maximum number of the peripheral edge is 6 and the maximum number of the component is 3. We also use additional path encoding in each layer. The learning rate is 0.001 and we use \textit{ReduceLROnPlateau} learning rate scheduler with patience of 5 and a reduction factor of 0.7. If the learning rate is less than $1e-6$, the training is stopped.

\textbf{ZINC datraset}: For ZINC dataset, we run the experiment 4 times independently and report the mean results. For each run, the maximum number of epochs is 500 and the batch size is 64. We implement the KP-GIN+ and KP-GIN$'$ for the ZINC dataset. For KP-GIN+, the hidden size is 104. The number of hops and the number of layers are both 8. For KP-GIN$'$, the hidden size is 96. The number of hops and the number of layers are 16 and 17 respectively. For both two models, we add the residual connection. The final node representation is the concatenation of each layer and the pooling method is the summation. We use the shortest path distance kernel. The initial learning rate is 0.001 and we use \textit{ReduceLROnPlateau} learning rate scheduler with patience of 10 and a reduction factor of 0.5. If the learning rate is less than $1e-6$, the training is stopped.
\section{Additional results}
\label{app:ablation}
In this section, we provide additional experimental results and discussion. The additional results on the counting substructure dataset are shown in Table~\ref{tab:ablation}. First, for K-GIN+, we observe a steady improvement for all tasks when we increase the K, which aligns with the theoretical results. Second, path encoding can hugely boost the performance of the counting substructure, which demonstrates the effectiveness of path encoding. Third, for KP-GIN+, most of the tasks achieve the best result with only K=1. This means we only need local peripheral subgraph information to count through substructures. Even though increasing the K would increase the expressive power monotonically from a theoretical point of view, it may add noise to the training process. Finally, path encoding does not show much effect on the KP-GNN model, which means information on path encoding is already encoded in the peripheral subgraph.

\begin{table*}[t]
\caption{Ablation study on counting substructure dataset. (* means add path encoding.)}
\vspace{-5pt}
\label{tab:ablation}
\begin{tabular}{@{}llllll@{}}

\toprule
\multicolumn{6}{c}{Counting substructures (MAE)} \\ \midrule
\multicolumn{1}{c}{model}                               & \multicolumn{1}{c}{K}  & \multicolumn{1}{c}{Triangle} & \multicolumn{1}{c}{Tailed Tri.} & \multicolumn{1}{c}{Star} & \multicolumn{1}{c}{4-Cycle} \\ \midrule
\multicolumn{1}{l|}{\multirow{4}{*}{\textbf{K-GIN+}}}   & \multicolumn{1}{l|}{1} & 0.4546 {\small$\pm$ 0.0107}              & 0.3665 {\small$\pm$  0.0004}                & 0.0412 {\small$\pm$ 0.0468}          & 0.3317 {\small$\pm$ 0.0121}             \\
\multicolumn{1}{l|}{}                                   & \multicolumn{1}{l|}{2} & 0.2938 {\small$\pm$ 0.0030}              & 0.2283 {\small$\pm$ 0.0012}                 & 0.0206 {\small$\pm$ 0.0095}          & 0.2330 {\small$\pm$ 0.0020}             \\
\multicolumn{1}{l|}{}                                   & \multicolumn{1}{l|}{3} & 0.2663 {\small$\pm$ 0.0088}             & 0.1998 {\small$\pm$ 0.0022}                 & 0.0174 {\small$\pm$ 0.0030}          & 0.2202 {\small$\pm$ 0.0061}             \\
\multicolumn{1}{l|}{}                                   & \multicolumn{1}{l|}{4} & 0.2593 {\small$\pm$ 0.0055}              & 0.1930 {\small$\pm$ 0.0033}                 & 0.0165 {\small$\pm$ 0.0041}          & 0.2079 {\small$\pm$ 0.0024}             \\ \midrule
\multicolumn{1}{l|}{\multirow{4}{*}{\textbf{K-GIN+*}}}  & \multicolumn{1}{l|}{1} & 0.4546 {\small$\pm$ 0.0107}              & 0.3665 {\small$\pm$ 0.0004}                 & 0.0412 {\small$\pm$ 0.0468}          & 0.3317 {\small$\pm$ 0.0121}             \\
\multicolumn{1}{l|}{}                                   & \multicolumn{1}{l|}{2} & 0.0132 {\small$\pm$ 0.0025}              & 0.0189 {\small$\pm$ 0.0049}                 & 0.0219 {\small$\pm$ 0.0045 }         & 0.0401 {\small$\pm$ 0.0027 }            \\
\multicolumn{1}{l|}{}                                   & \multicolumn{1}{l|}{3} & 0.0134 {\small$\pm$ 0.0020 }             & 0.0147 {\small$\pm$ 0.0017 }                & 0.0288 {\small$\pm$ 0.0062 }         & 0.0471 {\small$\pm$ 0.0033 }            \\
\multicolumn{1}{l|}{}                                   & \multicolumn{1}{l|}{4} & 0.0253 {\small$\pm$ 0.0085  }            & 0.0244 {\small$\pm$ 0.0028 }               & 0.0171 {\small$\pm$ 0.0035 }         & 0.0474 {\small$\pm$ 0.0025 }            \\ \midrule
\multicolumn{1}{l|}{\multirow{4}{*}{\textbf{KP-GIN+}}}  & \multicolumn{1}{l|}{1} & 0.0060 {\small$\pm$ 0.0008 }             & 0.0073 {\small$\pm$ 0.0020 }                & 0.0151 {\small$\pm$ 0.0022  }        & 0.2964 {\small$\pm$ 0.0080 }            \\
\multicolumn{1}{l|}{}                                   & \multicolumn{1}{l|}{2} & 0.0106 {\small$\pm$ 0.0015 }             & 0.0115 {\small$\pm$ 0.0017 }                & 0.0264 {\small$\pm$ 0.0064  }        & 0.0657 {\small$\pm$ 0.0034 }            \\
\multicolumn{1}{l|}{}                                   & \multicolumn{1}{l|}{3} & 0.0134 {\small$\pm$ 0.0020 }             & 0.0147 {\small$\pm$ 0.0017 }                & 0.0288 {\small$\pm$ 0.0062}          & 0.0471 {\small$\pm$ 0.0033 }            \\
\multicolumn{1}{l|}{}                                   & \multicolumn{1}{l|}{4} & 0.0125 {\small$\pm$ 0.0012}              & 0.0169 {\small$\pm$ 0.0028 }                & 0.0362 {\small$\pm$ 0.0113}          & 0.0761 {\small$\pm$ 0.0135 }            \\ \midrule
\multicolumn{1}{l|}{\multirow{4}{*}{\textbf{KP-GIN+*}}} & \multicolumn{1}{l|}{1} & 0.0068 {\small$\pm$ 0.0019}              & 0.0083 {\small$\pm$ 0.0019}                 & 0.0166 {\small$\pm$ 0.0041}          & 0.3063 {\small$\pm$ 0.0251}             \\
\multicolumn{1}{l|}{}                                   & \multicolumn{1}{l|}{2} & 0.0110 {\small$\pm$ 0.0016}              & 0.0110 {\small$\pm$ 0.0016 }                & 0.0255 {\small$\pm$ 0.0056}          & 0.0395 {\small$\pm$ 0.0018}             \\
\multicolumn{1}{l|}{}                                   & \multicolumn{1}{l|}{3} & 0.0117 {\small$\pm$ 0.0024}              & 0.0111 {\small$\pm$ 0.0025  }               & 0.0338 {\small$\pm$ 0.0077 }         & 0.0813 {\small$\pm$ 0.0136 }            \\
\multicolumn{1}{l|}{}                                   & \multicolumn{1}{l|}{4} & 0.0155 {\small$\pm$ 0.0037}              & 0.0168 {\small$\pm$ 0.0018 }                & 0.0422 {\small$\pm$ 0.0121 }         & 0.0538 {\small$\pm$ 0.0052 }            \\
\bottomrule
\end{tabular}
\end{table*}
\newpage
\section{Datasets Description and Statistics}\label{app:datasets}

\begin{table}[h]
% \color{blue}
% \revise{
% \footnotesize 
%\fontsize{8.7}{9.2}\selectfont
 \setlength{\tabcolsep}{1pt}
    \caption{Dataset statistics.}\label{tab:dataset}
    \centering
    \begin{tabular}{lcccc}
    \toprule
    Dataset &  \#Tasks & \# Graphs & Ave. \# Nodes & Ave. \# Edges  \\
    \midrule  
    EXP  &  2  & 1200 & 44.4 & 110.2\\
    SR25  &  15  & 15 & 25 & 300\\
    CSL & 10 & 150 & 41.0 & 164.0\\
    Graph\&Node property & 3 & 5120/640/1280 & 19.5 & 101.1\\
    Substructure counting & 4 & 1500/1000/2500 &18.8 &62.6\\

    \midrule  
    MUTAG     &  2 & 188 & 	17.93 & 19.79\\
    D\&D & 2 & 1178 & 284.32 & 715.66\\
    PTC-MR    &  2 & 344 & 14.29 & 14.69 \\
    PROTEINS  & 2 & 1113 & 39.06 & 72.82\\
    IMDB-B    &  2 & 1000 & 19.77 & 96.53\\
    \midrule
    QM9 & 12 & 129433 & 18.0 & 18.6 \\
    ZINC  &  1  & 10000 / 1000 / 1000& 23.1 & 49.8\\
    \bottomrule
    \end{tabular}
% }
\end{table}

\end{document}